\documentclass[twoside]{article}
\usepackage[accepted]{aistats2021}

\usepackage[round]{natbib}

\RequirePackage[loading]{tracefnt}

\usepackage[utf8]{inputenc} 
\usepackage{hyperref}       
\usepackage{url}            
\usepackage{booktabs}       
\usepackage{amsmath,amsfonts,amsthm,amssymb}       
\usepackage{nicefrac}       
\usepackage{microtype}      
\usepackage{wrapfig}

\hypersetup{
        colorlinks   = true,
        urlcolor     = blue,
        linkcolor    = blue,
        citecolor   = blue
}

\usepackage{fancyvrb}
\usepackage{enumitem}
\usepackage{algorithm}
\usepackage[capitalize]{cleveref}
\usepackage{algorithmicx}
\usepackage{algpseudocode}
\usepackage[nolist,nohyperlinks]{acronym}
\usepackage{color}
\usepackage{bold-extra}
\usepackage{mathtools}
\usepackage{etoolbox}
\usepackage{listings}
\usepackage[space=true]{accsupp}

\newcommand{\pmid}{|}
\renewcommand{\P}{\mathbb{P}}

\newcommand{\bA}{\boldsymbol{A}}

\newcommand{\bQ}{\boldsymbol{Q}}

\newcommand{\bx}{\boldsymbol{x}}

\newcommand{\bR}{\boldsymbol{R}}

\newcommand{\sY}{\mathcal{Y}}

\newcommand{\bbarW}{\boldsymbol{\bar{W}}}

\newcommand{\bZ}{\boldsymbol{Z}}

\newcommand{\bW}{\boldsymbol{W}}
\newcommand{\bV}{\boldsymbol{V}}
\newcommand{\bU}{\boldsymbol{U}}

\newcommand{\sX}{\mathcal{X}}

\newcommand{\sF}{\mathcal{F}}

\DeclareMathOperator*{\argmax}{arg\,max}

\newcommand{\Var}{\mathrm{Var}}

\newcommand{\btheta}{\boldsymbol{\theta}}
\newcommand{\bhtheta}{\boldsymbol{\hat{\theta}}}
\newcommand{\bTheta}{\boldsymbol{\Theta}}
\newcommand{\bhTheta}{\boldsymbol{\hat{\Theta}}}

\newcommand{\scO}{\mathcal{O}}

\newcommand{\wh}{\hat}
\newcommand{\ve}{\varepsilon}

\newtheorem{lemma}{Lemma}

\newtheorem{theorem}{Theorem}

\newtheorem{remark}{Remark}
\newtheorem{prop}{Proposition}

\newcommand{\reals}{\mathbb{R}}

\newcommand{\tp}{^{\top}}

\newcommand{\repk}{^{(k)}}

\DeclareMathOperator{\E}{\mathbb{E}}

\DeclareMathOperator{\Bias}{\mathrm{Bias}}

\newcommand{\ones}{\boldsymbol{1}}

\newtheoremstyle{named}{}{}{\itshape}{}{\bfseries}{}{.5em}{\thmnote{#3.}#1}
\theoremstyle{named}
\newtheorem*{nameddef}{}

\newcommand*\diff{\mathop{}\!\mathrm{d}}

\newcommand{\tr}{\mathrm{tr}}

\newcommand{\eps}{\epsilon}

\newcommand{\pr}[1]{\left( #1 \right)}
\newcommand{\br}[1]{\left[ #1 \right]}
\newcommand{\cbr}[1]{\left\{ #1 \right\}}
\newcommand{\abs}[1]{\left|#1\right|}

\newcommand{\wstar}{w^{\star}}

\newcommand{\vh}{\hat{v}}
\newcommand{\IS}{^{\text{\scshape{iw}}}}
\newcommand{\DR}{^{\text{\scshape{dr}}}}

\newcommand{\deli}{^{\backslash i}}

\newcommand{\delk}{^{\backslash k}}

\newcommand{\vhest}{\wh{v}^{\mathrm{est}}}

\newcommand{\WIS}{^{\text{\scshape{sn}}}}
\newcommand{\ISl}{^{\text{\scshape{iw}-}\lambda}}
\newcommand{\DRl}{^{\text{\scshape{dr}-}\lambda}}
\newcommand{\subIS}{_{\text{\scshape{iw}}}}
\newcommand{\subWIS}{_{\text{\scshape{sn}}}}

\newcommand{\conf}{x}

\newcommand{\bmid}{\;\middle|\;}

\newcommand{\MLE}{^{\text{\scshape{mle}}}}
\newcommand{\POP}{^{\text{\scshape{pop}}}}
\newcommand{\ESLB}{\text{\scshape{eslb}}}
\newcommand{\astar}{a^{\star}}

\begin{acronym}
\acro{RLS}{Regularized Least Squares}
\acro{ERM}{Empirical Risk Minimization}
\acro{RKHS}{Reproducing kernel Hilbert space}
\acro{DA}{Domain Adaptation}
\acro{PSD}{Positive Semi-Definite}
\acro{SGD}{Stochastic Gradient Descent}
\acro{OGD}{Online Gradient Descent}
\acro{SGLD}{Stochastic Gradient Langevin Dynamics}
\acro{IS}[IW]{Importance Weighting}
\acro{IPS}{Inverse Propensity Scoring}
\acro{WIS}[SN]{Self-normalized Importance Weighting}
\acro{EL}{Empirical Likelihood}
\acro{MGF}{Moment-Generating Function}
\acro{ES}{Efron-Stein}
\acro{ESS}{Effective Sample Size}
\acro{KL}{Kullback-Liebler}
\acro{DR}{Doubly-Robust}
\acro{ESLB}{Efron-Stein Lower Bound}
\acro{MLE}{Maximum Likelihood Estimator}
\acro{PAC}{Probably Approximately Correct}
\end{acronym}

\usepackage{xspace}

\begin{document}

\twocolumn[

\aistatstitle{Confident Off-Policy Evaluation and Selection through Self-Normalized Importance Weighting}

\aistatsauthor{Ilja Kuzborskij \And Claire Vernade \And Andr\'as Gy\"orgy \And Csaba Szepesv\'ari}

\aistatsaddress{ DeepMind } ]

\begin{abstract}
  We consider off-policy evaluation in the contextual bandit setting for the purpose of obtaining a robust off-policy \emph{selection} strategy, where the selection strategy is evaluated based on the value of the chosen policy in a set of proposal (target) policies. We propose a new method to compute a lower bound on the value of an arbitrary target policy given some logged data in contextual bandits for a desired coverage. The lower bound is built around the so-called \ac{WIS} estimator. It combines the use of a semi-empirical Efron-Stein tail inequality to control the concentration
  and a new multiplicative (rather than additive) control of the bias.
  The new approach is evaluated on a number of synthetic and real datasets and is found to be superior to its main competitors, both in terms of tightness of the confidence intervals and the quality of the policies chosen.
\end{abstract}

\section{Introduction}

Consider the following offline stochastic decision making problem where
an agent observes a collection of contexts, actions, and associated rewards collected by some \emph{behavior policy} and has to choose a new policy from a finite set of \emph{target policies}.
The agent's goal is to select the policy that has the highest \emph{value}, defined as its expected reward.
We call this variant of the contextual bandit problem \emph{best-policy selection} (in the off-policy setting).

This problem is encountered for instance in personalized recommendation and allocation problems, such as in medical applications, online advertising, and operations research: a \emph{static behavior policy} (e.g.\ a randomized classifier) is run online and for each chosen action, only partial (bandit) feedback is received. The collected data must then be used to evaluate other \emph{static target policies}~\citep{agarwal2017effective} with the goal of choosing a policy that will perform better on average.

We emphasize that this work focuses on such \emph{static} policies. These are an important class of policies that are preferred in practice in many cases. For example, recommender systems based on a batch-learnt classifier with predictable behavior or expert-designed rules as in medical applications.

At its core, this selection problem relies on off-policy evaluation~\citep{BoPe13,dudik2011doubly,swaminathan2015batch}, which is concerned with accurately estimating the value of a target policy, using a logged dataset, and aiming for a good bias-variance trade-off. 
To guarantee that such trade-off holds in practice, one would ideally rely on high-probability confidence bounds.
However, only few works on off-policy evaluation have provided practically computable, tight 
confidence bounds. 
It is recognized, though, that such bounds should depend on
the empirical variance of the estimator~\citep{BoPe13,thomas2015high_a,thomas2015high_b,swaminathan2015self,metelli2018policy}.
In general, this is a non-trivial task and the standard tools such as sub-Gaussian tail concentration inequalities (e.g.\ Bernstein's inequality) are ill-suited for this job.
Indeed, most estimators derive from \emph{\acf{IS}}, a standard technique for estimating a property of a distribution while having access to a sample generated by another distribution.
At the same time, arguably one of the most interesting scenarios is when the target and the behavior policies are misaligned, which corresponds to situations when the weights of \ac{IS} exhibit a heavy-tailed behavior.
In such cases, the control of the moments of the \ac{IS} estimator, and therefore its concentration, is in general futile.

In this paper we revisit \emph{\acf{WIS}}, a self-normalized version of \ac{IS}.
This estimator is asymptotically unbiased, and is known for its small variance in practice~\citep{hesterberg1995weighted}.
Moreover, unlike \ac{IS}, all the moments of \ac{WIS} are simultaneously bounded.
These favorable properties allow us to prove finite-sample concentration inequalities at the price of a (controllable) bias.
\paragraph{Contributions.} 
Our main result is
a new high-probability lower bound on the value of the \ac{WIS} estimator, stated in \cref{sec:bound} and proved in \cref{sec:proofs}.
Moreover, we formulate the off-policy selection problem and propose a systematic, appropriate approach using off-policy evaluation tools. 
In this context, we demonstrate empirically (\cref{sec:experiments}) that our bound achieves the best performance compared to all existing and proposed baselines.

\section{Notation and preliminaries}
\label{sec:setting}
\paragraph{Off-policy evaluation for contextual bandits. }
For the stochastic contextual bandit model, an off-policy evaluation problem is characterized by a triplet
$(P_X, P_{R|X,A}, \pi_b)$,
where $P_X$ is a probability measure over contexts (we assume that the context space is any probability space $(\sX, \Sigma_{\sX}, P_X)$),
$P_{R|X,A}$ is a probability kernel producing the reward distribution given the context $X \in \sX$ and action $A \in [K] = \cbr{1, \ldots, K}$,
and $\pi_b:\sX\to [K]$ is a behaviour policy, that is, a conditional distribution over actions given the context.

The decision maker observes $(S, \pi_b)$, where  $S = \pr{(X_1, A_1, R_1), \ldots, (X_n, A_n, R_n)}$ is a tuple of independent context-action-reward triplets,
obtained by following the behaviour policy $\pi_b$: for all $i\in [n]$, $A_i \sim \pi_b(\cdot |X_i)$, where $X_i\sim P_X$.
The reward $R_i \sim P_{R|X,A}$ is a \emph{bandit feedback} as it only reveals the value of the taken action $A_i$. For example, in a multi-class classification task, the reward may be a (noisy) binary random variable that indicates whether the chosen label is right or not, but the true label is not revealed.
We assume that the rewards are bounded in $[0,1]$
and that $\pi_b$
is known and can be evaluated at any context-action pair. In many applications the behavior policy represents the policy running in the system, which is usually known by the practitioner and can be queried.

A policy $\pi$ is any conditional distribution over the actions and its value $v(\pi)$ is defined by
\begin{equation}
v(\pi) = \int_{\sX} \sum_{a \in [K]} \pi(a | x) r(x,a) \diff P_X(x)
\end{equation}
where $r(x,a) = \int u \diff P_{R|X,A}(u|x,a)$ is the mean reward for a given context-action pair $(x,a)$.
Similarly to $\pi_b$, we assume that any known policy $\pi$ can be evaluated for any context-action pair.
In general, the goal of off-policy evaluation is to return an estimate $\vhest(\pi)$ of the value $v(\pi)$ of some \emph{target policy} with controlled bias and variance.
In contrast,
we are concerned with obtaining a data-dependent scoring function that allows the decision maker to 
choose the highest performing target policy in a set of candidate policies. We call this statistical problem \emph{best-policy selection}.
\paragraph{Best-policy selection. }
\begin{figure}
  \centering
  \includegraphics[width=8.3cm]{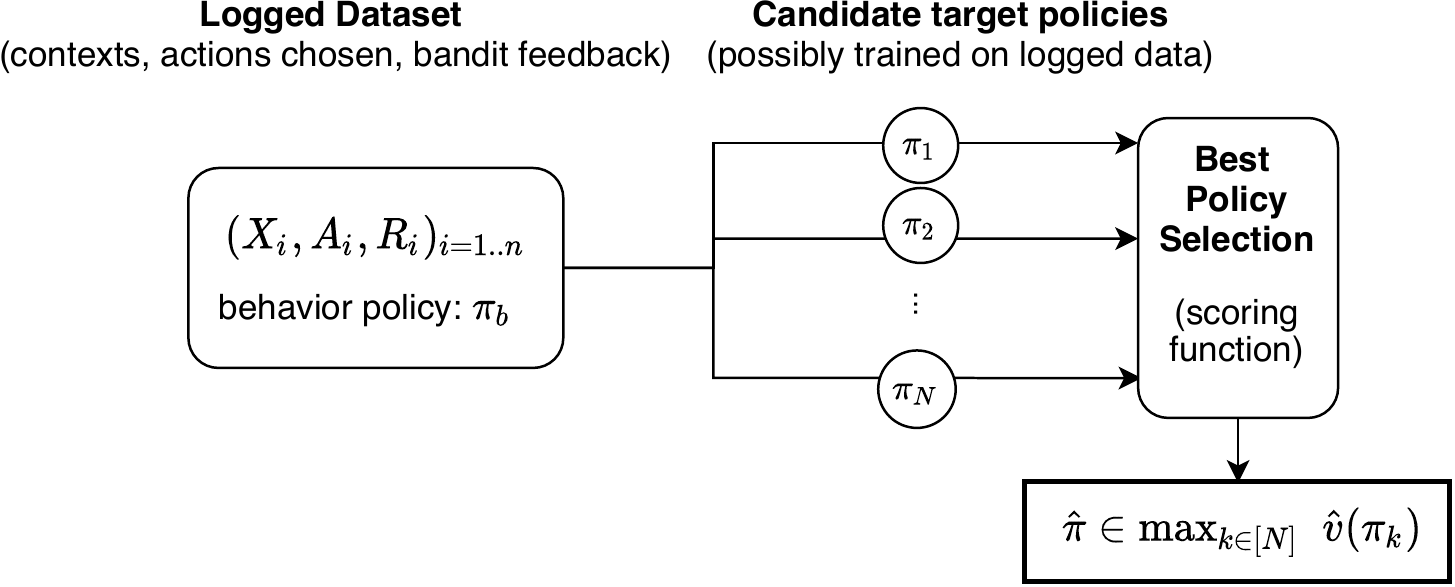}
  \caption{The best-policy selection problem}
  \label{fig:evaluator}
\end{figure}
Given a finite set $\{\pi_1,\ldots, \pi_N\}$ of 
policies, called the \emph{target} policies, our goal is to design a decision algorithm that returns a policy $\hat{\pi}$ with the highest value. We denote $\pi^* \in \max{_{k}}v(\pi_{k})$ a policy with maximum value in the target set, and the objective of the decision maker is to identify $\pi^*$.
The decision maker uses a scoring function $\vhest$ as input and chooses the policy that has the highest score:
$\hat{\pi} \in \argmax_{k} \vhest(\pi_{k})
$ (see Figure~\ref{fig:evaluator} for an illustration of the problem).
The quality of the selected $\hat{\pi}$ is measured by how close its value is to that of the optimal policy $\pi^*$.
The choice of the estimation method used to make the decision is crucial. A naive approach would be to use directly a value estimator (see below for a review of classical methods). However, we demonstrate in this paper that this may lead to dramatic losses. In turn, we propose a scoring function based on a high-probability lower bound on the value.

\begin{remark}
Best-policy selection is related to the problem of off-policy learning, which aims at designing data-based policies that have a high value. In our case, policies are created arbitrarily and possibly trained on the logged dataset, and our decision rule only guarantees that the decision maker returns the best performing one in a given set.
We pose the problem of learning over a discrete class as a decision problem.
\end{remark}

\paragraph{Classical estimation approaches. } \acf{IS} is the most widely known approach to obtain an unbiased value estimator:
$
  \vh\IS(\pi) = \frac1n \sum_{i=1}^n \frac{\pi(A_i|X_i)}{\pi_b(A_i|X_i)} R_i = \frac{1}{n} \sum_{i=1}^n W_iR_i~,
$
which is unbiased since $\E_{\pi_b}[WR]=\E_\pi[R]$.
Each data point is reweighted by its \emph{importance weight}, which is related to the likelihood ratio of the event of selecting the given action
under both policies.
In this paper we focus on a self-normalized version of \ac{IS}, called  \emph{\acf{WIS}} \citep{hesterberg1995weighted}:
\[
  \wh{v}\WIS(\pi) = \frac{\sum_{k=1}^n W_k R_k}{\sum_{i=1}^n W_i}~,
\]
where the sum of the weights (a random variable) is used instead of $n$ for normalization.
While the \ac{WIS} estimator is not unbiased, it is generally regarded as a good estimator.
As a start, we may for example note that it gives values in the $[0,1]$ interval.
To get a sense of the concentration of the \ac{WIS} estimator, it is instrumental to consider the bound that one may get from the Efron-Stein inequality on its variance (see, e.g. \citep{kuzborskij2019efron}): A quick calculation gives
$\text{Var}(\wh{v}\WIS(\pi)) \le n \E[\tilde W_1^2]$ where $\tilde W_1 = W_1/(W_1+\dots+W_n)$.
Note that in the ideal case when the behavior and target policies coincide, $W_i = 1$ and thus $\tilde W_1 =1/n$.
In the Monte-Carlo simulation literature, $n_{\text{eff}} = (\tilde W_1^2 + \dots + \tilde W_n^2)^{-1} \approx
1/(n \E[\tilde W_1^2])$ is known as the effective sample-size \citep{kong1992note,elvira2018rethinking}, and it
provides a quantitative complexity measure of the estimation problem for a target $\pi$: the quality of $\wh{v}\WIS(\pi)$ is as good as if we had used $n_{\text{eff}}$ samples from $\pi$ instead of $n$ samples from $\pi_b$.
However, since $n_{\text{eff}}$ is random, this is not entirely satisfactory as it does not lead to an easy finite-sample concentration bound using standard tools.

\section{Confidence bound for 
\acs{WIS} Estimator}
\label{sec:bound}
\label{sec:sim2}
We now state our main result, a high-probability lower bound on the value of a policy $\pi$ based on the \ac{WIS} estimator.
Even though we use it here as an efficient scoring function for the off-policy selection problem, we believe this result is of independent interest.
\begin{theorem}
  \label{thm:WIS_contextual_sim_lb_value}
  Let $W_i=\pi(A_i|X_i) / \pi_b(A_i|X_i)$ for all $i$, and assume that
  $(W_i, R_i)_{i=1}^n$ are independent from each other, reward $R_i \in \{0,1\}$ a.s. (almost surely) is `one-hot'.\footnote{Reward is `one-hot' when for any context $x$, there exists a unique action $\astar$, such that $r(x, \astar) = 1$, and $r(x, a) = 0$ for $a \neq \astar$.}
  Let $Z = W_1 + \dots + W_n$, $Z\delk = Z - W_k$, and $Z\repk = Z\delk + W_k'$.
  Then, for any $x\ge 2$, with probability at least $1-2e^{-x}$,
  \begin{align*}
    v(\pi)
      &\geq
      \pr{
      B
      \pr{\wh{v}\WIS(\pi) - \eps}_+
      - \sqrt{\frac{x}{2 n}}
    }_+
        \mathrm{where}
  \end{align*}
  \begin{align*}
    \eps &= \sqrt{2 \pr{V\WIS + U\WIS} \pr{x + \frac12 \ln\pr{1 + \frac{V\WIS}{U\WIS}}}}\\
    V\WIS &= \sum_{k=1}^n \E\br{
    \pr{\frac{W_k}{Z} + \frac{W_k'}{Z\repk}}^2
    \bmid W_1^k, X_1^n }~,\\
    U\WIS &= \E[V\WIS \mid X_1^n]\\
    B &= \min\pr{\E\br{\frac{n}{\min_k Z\delk} \bmid X_1^n}^{-1}, 1}
  \end{align*}
  where $(a)_+ = \max(a, 0)$ and $a_1^l=(a_1,\ldots,a_l)$.
\end{theorem}

\textbf{Comments. }
This bound essentially depends on an
\emph{Efron-Stein estimate} of the variance of \ac{WIS}, $V\WIS$, while always conditioning on the contexts.
It can be qualified of \emph{semi-empirical} as it relies on taking expectations over the weights.
Thus, its computability relies on our ability to compute those expectations, which requires knowing $\pi_b$ and being able to evaluate it on any context-action pair.
Remarkably, $V\WIS$ can be computed without \emph{any} knowledge of the context distribution $P_X$ or the reward probability kernel $P_{R|X,A}$.
Notably, the bias $B$ quantifying policy mismatch appears as a \emph{multiplicative} term in the bound, rather than additive. Indeed, when $\pi$ coincides with $\pi_b$ exactly, $B=1$ and the estimator suffers no bias. Conversely, the more the mismatch, the larger the bias is.
Similarly to $V\WIS$,
the bias can be computed since the distribution of the importance weights
(conditioned on the contexts)
is known.
Finally, as discussed earlier, it is worth noting that the behavior of $V\WIS$ is closely related to that of the effective sample size $n_\text{eff}$: when the target and the behavior policies coincide, we have $V\WIS = \scO(1/n)$, while when policies are in full mismatch (e.g.\ Dirac deltas on two different actions) we have $V\WIS =\Theta(1)$.
In the intermediate regime of a partial mismatch, for instance when the distribution of importance weights is heavy-tailed, $V\WIS = o(1)$ and thus $\eps \to 0$ as $n \to \infty$.

A `one-hot' assumption on rewards might appear strong at first glance, however it is not unrealistic because it corresponds to the multiclass classification with bandit feedback.

Theorem~1 generalizes and strengthens a similar result on the \ac{WIS} estimator developed in~\cite[Theorem~7]{kuzborskij2019efron}. Compared to the latter work, there are two key differences:
First, our bound is designed for the contextual bandit setting as opposed to the context-free setting considered in~\citep{kuzborskij2019efron}. Second, our bound is significantly tighter, due to a novel multiplicative control of the bias of the estimator.
\textbf{Computation of the bound. }
Although the bound can be computed exactly, computing expectations can be expensive for large action spaces (exponential in $K$).
In this paper we compute an approximation to all expectations in $V\WIS$, $U\WIS$, and $B$ using Monte-Carlo simulations, as presented in  \cref{alg:v_WIS2}, and its Python implementation is presented in \cref{sec:listing}.
The algorithm simply updates averages, $\tilde{V}\WIS_t$, $\tilde{U}\WIS_t$, and $\tilde{B}_t$,
over rounds $t=1,2,\ldots$,
where in each round, a fresh tuple of weights $(W_i')_i$ is sampled from $\pi_b(\cdot | X_1) \times \dots \times \pi_b(\cdot | X_n)$.
The simulation needs to be run until we obtain good enough estimates for the quantities $V\WIS$,$U\WIS$, and $B$. This can be checked via standard empirical concentration bounds on the estimation errors, such as the empirical Bernstein's inequality~\citep{mnih2008empirical,maurer2009empirical}. For example, denoting the sample variance of 
$\tilde{V}\WIS_t$ by $\widehat{\Var}(\tilde{V}\WIS_t)$, stopping the simulation when
\begin{equation}
    \label{eq:convergence_V}
     \sqrt{2 \widehat{\Var}(\tilde{V}\WIS_t)/t} + 14x/(3t -3) \le \ve
\end{equation}
holds, guarantees that
the simulation error $|V\WIS - \tilde{V}\WIS_t|$ is bounded by $\ve$ w.p.\ at least $1-e^{-x}$ (the stopping conditions for the other quantities are deferred to \cref{sec:appendix:stopping}).\footnote{The parameter $\ve$ needs to be specified by the user. A typical choice is $\ve = 1/n$ since $V\WIS \geq 1/n$ a.s.}
Then, denoting by $T_{\ve}$ the number of iterations until stopping, we can use \cref{thm:WIS_contextual_sim_lb_value} with $V\WIS$ replaced by $\tilde{V}\WIS_{T_{\ve}} + \ve$.
Of course, each application of the convergence tests
needs to be combined with \cref{thm:WIS_contextual_sim_lb_value} through a union bound: for example, verifying the convergence for all three variables every $2^k$ steps for $k = 1,2,\ldots$ means that the final bound on the value holds w.p.\ at least $1-(2+3 \log_2(T_{\ve})) e^{-x}$.

\begin{algorithm}[t]
  \caption{Computation of estimates for a variance proxy $V\WIS, U\WIS$. Scalar operations are understood pointwise when applied to vectors.}
  \begin{algorithmic}[1]
    \Require{observed context-action pairs $S = (X_i, A_i))_{i=1}^n$, behavior / target policy $\pi_b$  / $\pi$
    }
    \Ensure{Estimate of a variance proxy $\tilde{V}\WIS$ to be used in~\cref{thm:WIS_contextual_sim_lb_value}}
    \State $\bW \gets \br{\frac{\pi(A_1 \pmid X_1)}{\pi_b(A_1 \pmid X_1)}, \ldots, \frac{\pi(A_n \pmid X_n)}{\pi_b(A_n \pmid X_n)}}$
    \State $\bbarW \gets \br{W_1, W_1 + W_2, \ldots,
    W_1 + \dots + W_n}$
    \State $\bV, \bU \gets [0, \ldots, 0] \in \reals^n$, $\tilde{Z}^{\text{inv}} \gets 0$
    \State $t \gets 1$
    \Repeat
    \State $\bW' \gets \br{\frac{\pi(A_1' \pmid X_1)}{\pi_b(A_1' \pmid X_1)}, \ldots, \frac{\pi(A_n' \pmid X_n)}{\pi_b(A_n' \pmid X_n)}}$
    \Statex \hspace{4mm} where $\bA' \sim \pi_b(\cdot \pmid X_1) \times \dots \times \pi_b(\cdot \pmid X_n)$
    \State Sample $\bW''$, $\bW'''$ as independent copies of $\bW'$
    \State $\bbarW^{'\text{rev}} \gets \br{\sum_{i=1}^n W_i', \sum_{i=2}^n W_i', \ldots, W_n'}$
    \State $\bZ \gets \bbarW_{1:n-1} + \bbarW^{'\text{rev}}_{2:n}$ \Comment{\parbox{3cm}{Partially simulated sums of weights}}
    \State $\bV \gets (1 - \frac1t)\bV + \frac1t \pr{\frac{\bW}{\bZ}}^2$
    \State $\bU \gets (1 - \frac1t)\bU + \frac1t \pr{\frac{\bW''}{\sum_j W_j''}}^2$
    \State $\tilde{Z}^{\text{inv}}_t \gets (1 - \frac1t) \tilde{Z}^{\text{inv}}_t + \frac1t \cdot \frac{1}{\min_k \sum_{j \neq k} W_j'''}$
    \State $\tilde{V}\WIS_t \gets \bV \boldsymbol{1}$, $\tilde{U}\WIS_t \gets \bU \boldsymbol{1}$
    \State $t \gets t + 1$
    \Until{Convergence of $\tilde{V}\WIS_t$, $\tilde{U}\WIS_t$, and $\tilde{Z}_t$\\ \qquad \ \ (see main text and \cref{eq:convergence_V})}
    \State $\tilde{B}_t \gets \min\cbr{1, \frac{1}{n \tilde{Z}^{\text{inv}}_t}}$
    \State \textbf{return} $\tilde{V}\WIS_t$, $\tilde{U}\WIS_t$, $\tilde{B}_t$
  \end{algorithmic}
  \label{alg:v_WIS2}
\end{algorithm}

\section{Proof of \cref{thm:WIS_contextual_sim_lb_value}}

\label{sec:proofs}
We start with the decomposition of $v(\pi) - \wh{v}\WIS(\pi)$ as
\begin{align*}  
  \underbrace{
  v(\pi) - \E\br{v(\pi) \,|\, X_1^n}
  }_{\text{Concentration of contexts}}
  &+
  \underbrace{
  \E\br{v(\pi) \,|\, X_1^n} - \E\br{\wh{v}\WIS(\pi) \mid X_1^n}
  }_{\text{Bias}} \\
  & +
  \underbrace{
  \E\br{\wh{v}\WIS(\pi) \mid X_1^n} - \wh{v}\WIS(\pi)
  }_{\text{Concentration}}~.
\end{align*}
Each paragraph below focuses respectively on the \emph{Concentration}, \emph{Bias} and \emph{Concentration of contexts} term.
\paragraph{Concentration.}
We use a conditioned form of the concentration inequality of \citet[Theorem 1]{kuzborskij2019efron} that we restate below without a proof.\footnote{The proof of \cite{kuzborskij2019efron} can be applied almost exactly with minimal, trivial changes.}
The form of the result is slightly different from its original version to better suit our needs: the version stated here uses a filtration and eventually we use this with $\sF_0$ defined as the $\sigma$-algebra generated by the contexts $X_1^n$.
\begin{theorem}
  \label{thm:concentration}
  Let $(\sF_i)_{i=0}^n$ be a filtration and let $S = \pr{Y_1, \ldots, Y_n}$ be a sequence of random variables such that the components of $S$ are independent given $\sF_0$ and
  $(Y_k)_k$ is $(\sF_k)_k$-adapted.
  Then, for any $\conf \geq 2$ and $y > 0$ we have with probability at least $1-e^{-\conf}$, 
  \begin{align*}
      \lefteqn{|f(S) - \E[f(S)\mid \sF_0]|} \\
 & \leq \sqrt{2 (V + y) \pr{x + \ln\pr{\sqrt{1 + V / y}}}}
  \end{align*}
  where $V = \E\br{\sum_{k=1}^n (f(S) - f(S\repk))^2 \bmid Y_1, \ldots, Y_k}$ with $S\repk$ being $S$ with its $k$th element replaced with an independent copy of $Y_k$.
\end{theorem}
We apply the inequality with  $f = \vh\WIS$, $S = \pr{(W_1, R_1), \ldots, (W_n, R_n)}$
and $\sF_k$ being the $\sigma$-algebra generated by $X_1^k$; then $((W_k,R_k))_k$ is $(\sF_k)_k$-adapted, and taking $y=\E[V\WIS | X_1^n]$, we get that for any $x\geq 2$, w.p.\ at least $1-e^{-x}$,
\begin{equation}
\label{eq:eps_concentration}
 \E[\wh{v}\WIS(\pi) \mid X_1^n] - \wh{v}\WIS(\pi) \geq - \eps
\end{equation}
where $\eps$ is defined in  \cref{thm:WIS_contextual_sim_lb_value}, and we also used that
$ V \leq V\WIS$ (see \cref{prop:V_WIS_loo} in~\cref{sec:proof_from_the_main_text}).
\paragraph{Bias.}
Now we turn our attention to the \emph{bias} term.\footnote{An original version of the paper contained a mistake in the proof of the bias term. Here we present a corrected version with all experimental results remaining unchanged --- see \cref{sec:errata} for details.}
Let $v(\pi\pmid x)$ denote the value of a policy given a fixed context $x \in \sX$:
$
  v(\pi \pmid x) = \sum_{a \in [K]} \pi(a \pmid x) r(x,a)
$.
Then, since $A_k \sim \pi_b(\cdot \pmid X_k)$,
\begin{align}
  &\E\!\br{\sum_{k=1}^n W_k R_k \!\bmid \!X_1^n}
  =
    \sum_{k=1}^n \E\!\br{\frac{\pi(A_k \pmid X_k)}{\pi_b(A_k \pmid X_k)} \, R_k \!\bmid\! X_k}\\
  &=
    \sum_{k=1}^n \sum_{a \in [K]}\! \pi(a \pmid X_k) r(X_k,a)
  =
    \sum_{k=1}^n v(\pi \pmid X_k)~. \label{eq:bias}
\end{align}
To relate the above to the expectation of an \ac{WIS} estimator to the value we use the following lemma (proof is deferred to \cref{sec:proof_from_the_main_text}).
\begin{lemma}
  \label{lem:one_hot_id}
  Suppose that reward function
  is `one-hot', that is, for any context $x$, $r(x, \astar)=1$ for some unique action $\astar$ and $r(x, a)=0$ for all $a \neq \astar$.
  Denote $\wstar_i = \frac{\pi(\astar \mid X_i)}{\pi_b(\astar \mid X_i)}$.
  Then,
\begin{align*}
  \E\br{\wh{v}\WIS(\pi) \bmid X_1^n}
  =
  \sum_{i=1}^n v(\pi \pmid X_i) \E\br{\frac{1}{w^{\star}_i + Z\deli} \bmid X_1^n}~.
\end{align*}
\end{lemma}
And this implies
\begin{align*}
  \E\br{\wh{v}\WIS(\pi) \bmid X_1^n}
  \leq
  \pr{\sum_{i=1}^n v(\pi \pmid X_i)} \E\br{\frac{1}{\min_i Z\deli} \bmid X_1^n}~.
\end{align*}

\paragraph{Concentration of contexts.}
All that is left to do is to account for the randomness of contexts.
Since $(v(\pi \pmid X_k))_{k \in [n]}$ are independent and they take values in the range $[0,1]$, 
by Hoeffding's inequality we have for $x \geq 0$, w.p.\ at least $1-e^{-x},$ that
$
  \sum_{k=1}^n v(\pi \pmid X_k) - n v(\pi) \leq \sqrt{n x / 2}~.
$
Hence we bound the bias term as:
{\small
\begin{align*}
  \lefteqn{v(\pi) - \E\br{\wh{v}\WIS(\pi) \mid X_1^n}} \\
   &\geq
    v(\pi) - \E\br{\frac{1}{\min_k Z\delk} \bmid X_1^n} \sum_{k=1}^n v(\pi \pmid X_k)\\
      &\geq
    v(\pi) \pr{1 - \E\br{ \frac{n}{\min_k Z\delk} \bmid X_1^n}} - \E\br{\frac{\sqrt{n x / 2}}{\min_k Z\delk} \bmid X_1^n}
\end{align*}
}
Combining the bias bound above with the concentration term \cref{eq:eps_concentration} through the union bound we get, w.p.\ at least $1-2 e^{-\conf}$, that
\begin{align*}
  v(\pi) - \wh{v}\WIS(\pi)
  &\geq
    v(\pi) \pr{1 - \E\br{ \frac{n}{\min_k Z\delk} \bmid X_1^n}}\\
    &- \E\br{\frac{\sqrt{n x / 2}}{\min_k Z\delk} \bmid X_1^n} - \eps~.
\end{align*}
Noticing that $v(\pi) \geq 0$ and rearranging gives
\begin{align*}
  v(\pi)
  &\geq
    \pr{
    \underbrace{\E\br{ \frac{n}{\min_k Z\delk} \bmid X_1^n}^{-1}}_{B'}
    \pr{\wh{v}\WIS(\pi) - \eps}
    - \sqrt{\frac{x}{2n}}
    }_+ \\
  & \geq
    \pr{
    \min\cbr{1, B'}
    \pr{\wh{v}\WIS(\pi) - \eps}_+
    - \sqrt{\frac{x}{2n}}
    }_+
\end{align*}%
\noindent because
$(a b - c)_+ = (a (b)_+ - c)_+ \geq (a' (b)_+ - c)_+$ for $a \geq a' \geq 0, b \in \reals, c \geq 0$.
The proof of \cref{thm:WIS_contextual_sim_lb_value} is now complete.

\section{Related work and baseline confidence intervals}
\label{sec:baselines}

The benefits of using confidence bounds in off-policy evaluation and learning has been recognized in a number of works~\citep{BoPe13,thomas2015high_a,swaminathan2015batch,swaminathan2015self}.
Arguably, the standard tool in off-policy evaluation is the \acs{IS} estimator that originates from the sampling literature~\citep{owen2013book}.
However, it has a high variance when the weights have a heavy-tailed distribution. There has been many attempts to stabilize this estimator, including truncation \citep{ionides2008truncated,thomas2015high_b, BoPe13} or smoothing \citep{vehtari2015pareto}.
These more stable estimators admit confidence intervals manifesting a bias-variance trade-off, but tuning the truncation or smoothing process is a hard problem on its own, lacking good practical solutions, as discussed in \citep{gilotte2018offline}.
One approach is to tune the level of truncation depending on the data (e.g.\ by looking at importance weight quantiles) \citep{BoPe13}, however, this does not guarantee that the resulting estimator is unbiased.
Another popular technique is tuning the truncation level on a hold-out sample \citep{thomas2015high_a,swaminathan2015counterfactual}.
A closely related approach to truncation is smoothing, and in this paper we compare against an asymptotically unbiased, smoothed version of \acs{IS} (tuned in a data-agnostic way), described in \cref{sec:baseline_bounds}.

In contrast, \acs{WIS}, has a low variance in practice and good concentration properties even when the distribution of the weights is (moderately) heavy-tailed.
Asymptotic concentration results were already mentioned by \cite{hesterberg1995weighted} and polynomial (low-probability) finite-time bounds were explored by \cite{metelli2018policy}.
An alternative source of variance in \ac{IS} is due to the randomness of the rewards; a popular method mitigating its effect is the so-called \ac{DR} estimator~\citep{dudik2014doubly}, further improved by \cite{farajtabar2018more}, and stabilized by truncations in \citep{wang2017optimal,su2019cab, su2019doubly}. \ac{DR} and \ac{IS} can be more generally and optimally mixed as in \citep{kallus2018balanced} who prove asymptotic MSE error bounds.
Unfortunately, to the best of our knowledge, many of those works present ``sanity-check'' bounds (e.g., verifying asymptotic lack of bias), which are practically uncomputable for problems like ours.
A notable exception is the family of bounds with the aforementioned truncation: in this paper, for completeness, we present finite-sample confidence bounds for such stabilized estimators (see Section~\ref{sec:baseline_bounds}), which we use as additional baselines.

Finally, a somewhat different approach compared to all of the above was recently explored in \citep{karampatziakis2019empirical} based on the \emph{empirical likelihood} approach.
Their estimator and the corresponding lower bound on the value relies on solving a convex optimization problem.
However, in contrast to the above works, their bound only holds asymptotically, that is, in probability as $n \to \infty$.
In the following section we discuss it in detail and compare in the forthcoming experiments.

\subsection{Baseline confidence intervals}
\label{sec:baseline_bounds}
We derive high-probability lower bounds for the stabilized \ac{IS}-$\lambda$ and \ac{DR}-$\lambda$ estimators using the same technique as above.
Proofs for all statements in this section are given in \cref{sec:additional_proofs}.

\textbf{\ac{IS}}-$\lambda$.
Truncation of importance weights is a standard stabilization technique used to bound moments of the \ac{IS} estimator \citep{BoPe13,swaminathan2015self}.
Here we focus on a closely related (albeit theoretically more appealing due to its smoothness) $\lambda$-\emph{corrected} version of \ac{IS}, $\vh\ISl$, where instead of truncation we add a corrective parameter to the denominator of the importance weight, that is,
$
W_i^{\lambda} = {\pi(A_k|X_k)}/\pr{\pi_b(A_k|X_k) + \lambda}
$
for some $\lambda > 0$ (note that $W_i^{\lambda} \leq \min(W_i, 1/\lambda)$).
This ensures that weights are bounded, and setting $\lambda = 1 / \sqrt{n}$ ensures that the estimator is asymptotically unbiased.
Exploiting this fact, we prove the following confidence bound based on the empirical Bernstein's inequality \citep{maurer2009empirical} (also presented in \cref{sec:IS_lambda} for completeness).
\begin{prop}
  \label{prop:IS_lambda_bernstein}
  For the \ac{IS}-$\lambda$ estimator we have with probability at least $1-3e^{-\conf}$, for $\conf > 0$,
  \begin{align*}
    v(\pi)
    &\geq
      \wh{v}\ISl(\pi)
      -\sqrt{\frac{2x}{n} \, \Var\ISl(X_1^n)}
      - \frac{7 x}{3 \lambda (n-1)} \\
      &- \Bias\ISl(X_1^n)
      - \sqrt{\frac{x}{2 n}}
  \end{align*}
  where $\Var\ISl$ and $\Bias\ISl$ are, respectively, the empirical variance and bias of the estimator, defined in the full statement of the proposition in \cref{sec:IS_lambda}.
\end{prop}

\textbf{DR}-$\lambda$ \citep{dudik2011doubly, farajtabar2018more, su2019doubly} combines a direct model estimator and \ac{IS}, finding a compromise that should behave like \ac{IS} with a reduced variance.
As in the case of \ac{IS}, we introduce a $\lambda$-corrected version of \ac{DR}, $\vh\DRl$ (given formally in \cref{sec:DR_lambda}), where importance weights are replaced with $W_i^{\lambda}$.
This allows to prove the following bound:
\begin{prop}
  \label{prop:DR_lambda_bernstein}
For the \ac{DR}-$\lambda$ estimator defined w.r.t.\ a fixed $\eta : \sX \times [K] \to [0,1]$ we have with probability at least $1-3e^{-\conf}$, for $\conf > 0$,
  \begin{align*}
  v(\pi)
  &\geq
    \wh{v}\DRl(\pi)
    -\sqrt{\frac{2x}{n} \, \Var\DRl(X_1^n)} \\
    &- \frac{7}{3} \pr{1 + \frac{1}{\lambda}} \frac{x}{n-1} 
     - \Bias\DRl(X_1^n)
    - \sqrt{\frac{x}{2 n}}~.
  \end{align*}
where $\Var\DRl$ and $\Bias\DRl$ are, respectively, the variance and bias estimates defined in \cref{sec:DR_lambda}.
\end{prop}
As before, setting $\lambda = 1/\sqrt{n}$ ensures that 
\ac{DR}-$\lambda$
is asymptotically unbiased.

\textbf{Chebyshev-\ac{WIS}.} A Chebyshev-type confidence bound for \ac{WIS} can also be proved relying on the fact that the moments of \ac{WIS} are bounded. We present this result as a (naive) alternative approach to the more involved one proposed in this work. This idea was explored in the context of Markov decision processes by~\cite{metelli2018policy}.
\begin{prop}
  \label{prop:wis_cheb}
  With probability at least $1-3 e^{-x}$ for $x > 0$,
  \begin{align*}
    &v(\pi) \geq
    \frac{N_x}{n} \pr{
      \vh\WIS(\pi)
      - \sqrt{\frac{\sum_{k=1}^n \E[W_k^2 | X_k]}{N_x^2} \, e^x}
    }
    - \sqrt{\frac{x}{2 n}}~,\\
    ~
    &\text{where}
    ~
    N_x = n - \sqrt{2 x \sum_{k=1}^n \E\br{W_k^2\mid X_1^n}}~.
  \end{align*}
\end{prop}

\textbf{\ac{EL} esimator.}
The EL estimator for off-policy evaluation introduced by \cite{karampatziakis2019empirical} comes with \emph{asymptotic} confidence intervals: for any error probability $\delta>0$, the coverage probability of the confidence interval tends to $1-\delta$ as the sample size $n \to \infty$.
In particular, the EL estimator is based on a \ac{MLE} $\vh\MLE(\pi) = \bW\tp \bQ\MLE \bR$ where $\bW = [W_1, \ldots, W_n]\tp$, $\bR = [R_1, \ldots, R_n]\tp$, and the matrix $\bQ\MLE$ is a solution of some (convex) empirical maximum likelihood optimization problem which can be solved efficiently.
The empirical likelihood theory of \cite{owen2013book} provides a way to get a slightly different estimator that comes with asymptotic confidence intervals. 
More precisely, assuming that $\bQ\POP \succeq 0$ is some matrix satisfying $v(\pi) = \bW\tp \bQ\POP \bR$, the asymptotic theorem of empirical likelihood~\citep{owen2013book} gives an asymptotic identity
\[
  \lim_{n \to \infty} \P\pr{\tr(\ln \bQ\MLE) - \tr(\ln \bQ\POP) \leq \tfrac12 \chi^2_{q:1-\delta} } = 1 - \delta
\]
for any error probability $\delta$ where $\chi^2_{q:1-\delta}$ is the $1-\delta$ quantile of a $\chi^2$ distribution with one degree-of-freedom.
Based on this \citep{karampatziakis2019empirical}  estimated the value $v(\pi)$ by solving the optimization problem
\begin{align*}
  \min_{\bQ \succeq 0} \bW\tp \bQ \bR \quad s.t.\
  &\bW\tp \bQ \ones = 1, \ones\tp \bQ \ones = 1,\\
  &x_n + \tr(\ln \bQ) \geq \tr(\ln \bQ\MLE)
\end{align*}
where $x_n$ is a specific sequence satisfying $x_n \to \chi^2_{q:1-\delta}$ as $n \to \infty$.
In our experiments we use the implementation of the authors of \cite{karampatziakis2019empirical} available in \citep{mineiro2019el_code}. More precisely, we use their function \texttt{asymptoticconfidenceinterval}, which returns the lower and upper bounds of the confidence interval to run our experiments (we use only the lower bound as returned value).

\section{Experiments}
\label{sec:experiments}

\begin{figure*}[!t]
    \centering
    \includegraphics[width=0.245\linewidth]{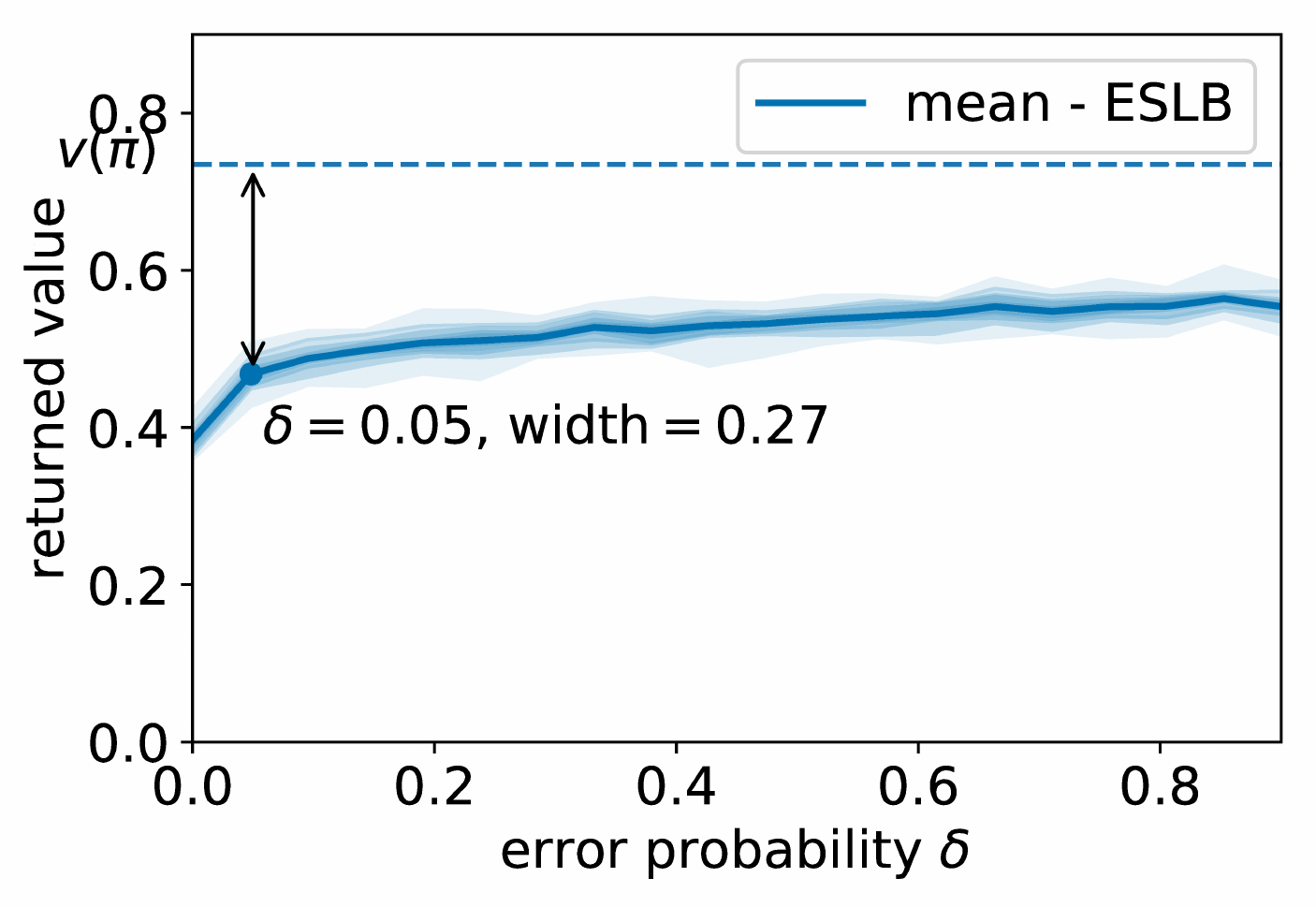}
    \includegraphics[width=0.245\linewidth]{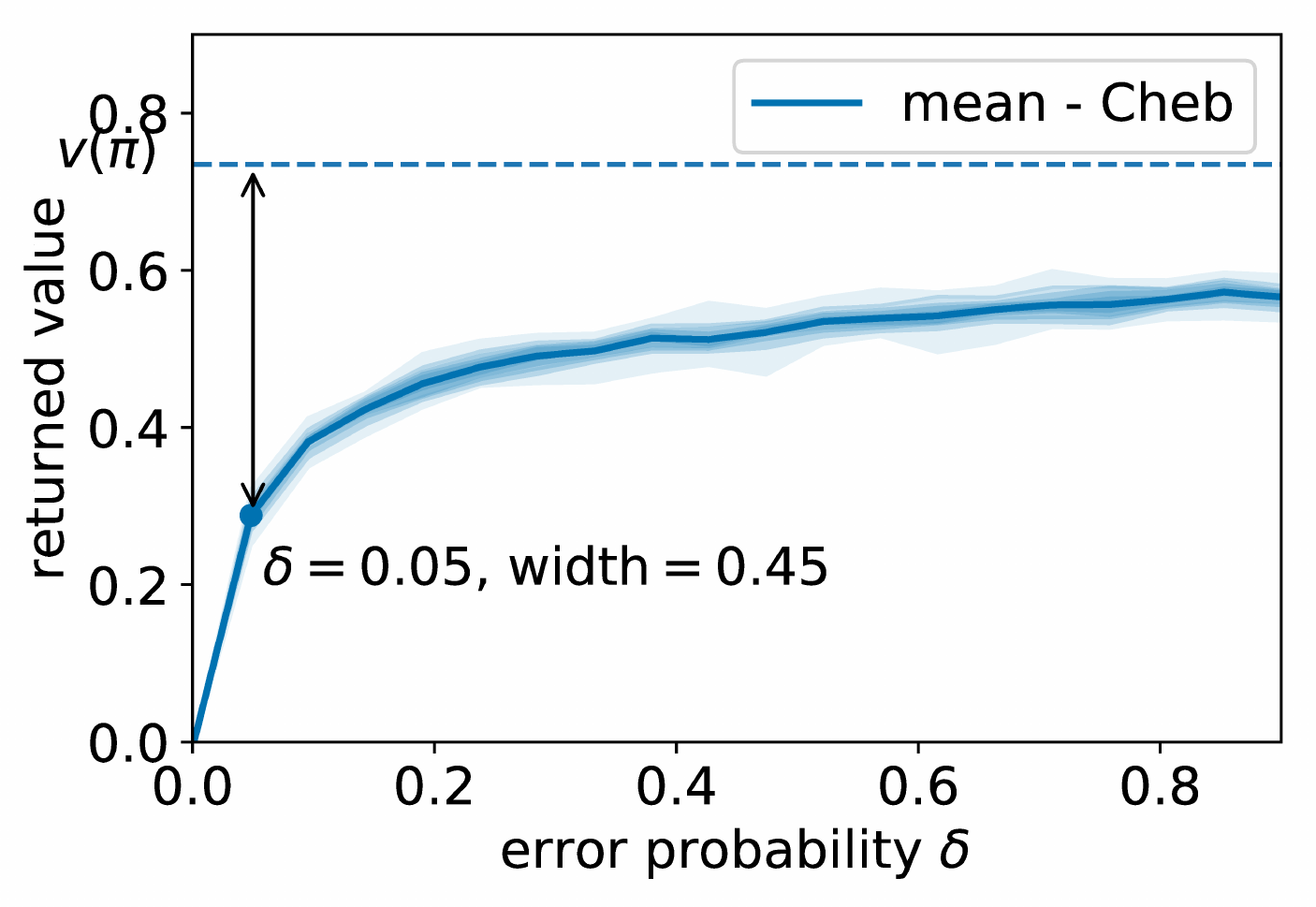}
    \includegraphics[width=0.245\linewidth]{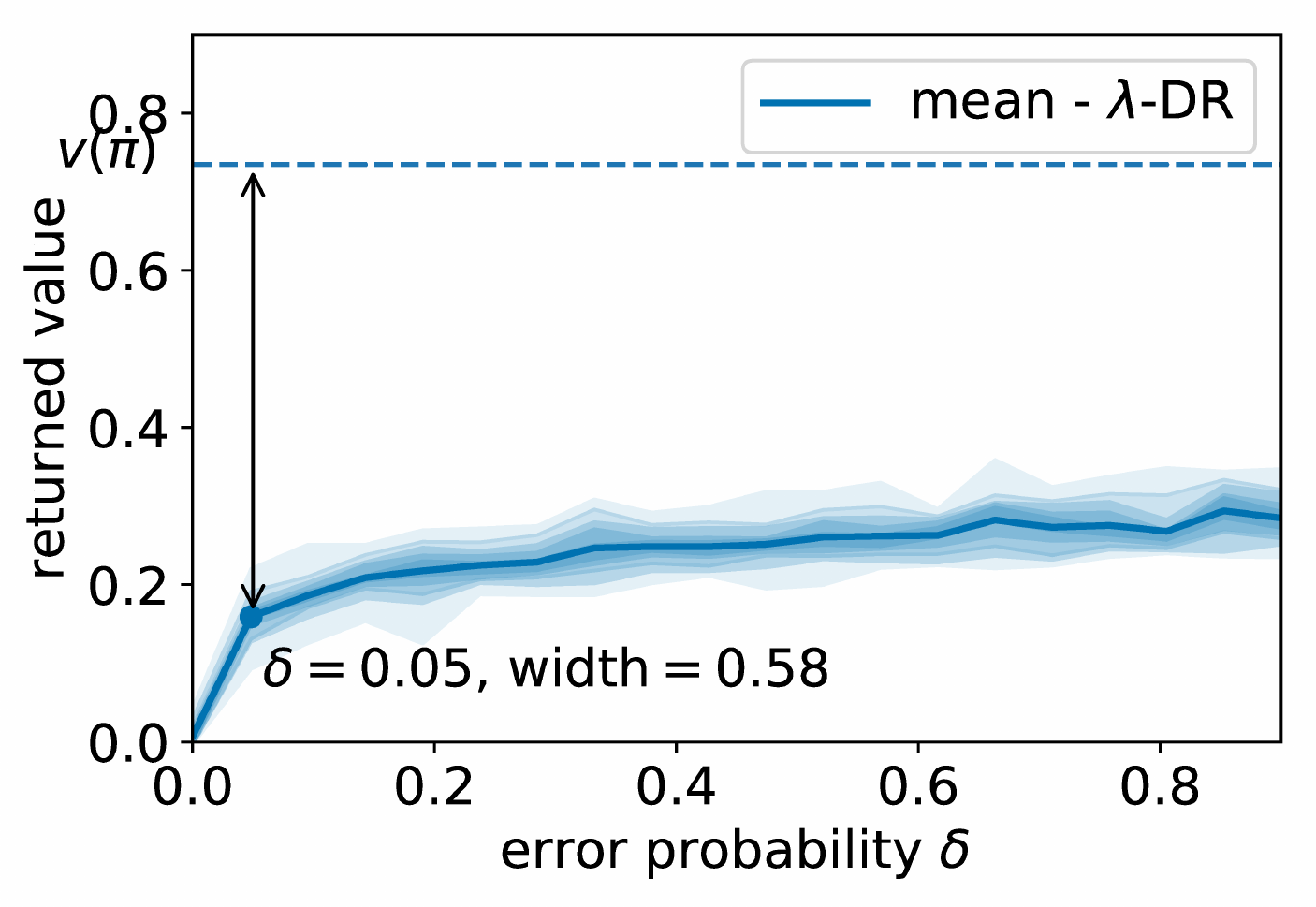}
    \includegraphics[width=0.245\linewidth]{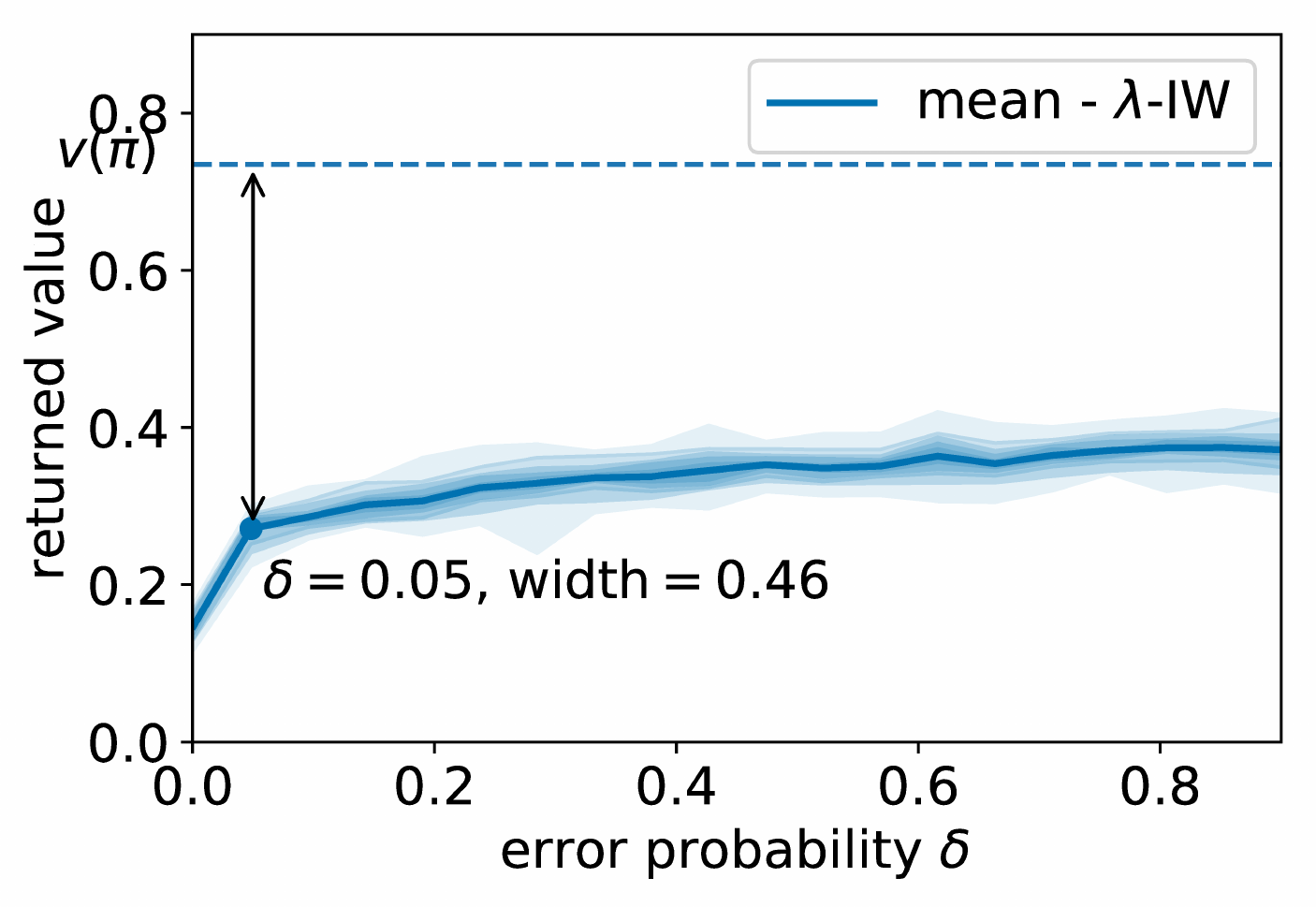}
    \caption{Empirical tightness: analysis on synthetic data. From left to right: \ac{ESLB}, Chebyshev, $\lambda$-\ac{DR} and $\lambda$-\ac{IS}. 100 runs for each value of $\delta$.}
    \label{fig:coverage-analysis}
\end{figure*}

\begin{table*}[!t]
  \caption{Average test rewards for a $5$-action problem on a synthetic benchmark.
    Symbol $-\infty$ indicates that no policy can be selected, since the confidence bound is always vacuous.
    Here the behaviour policy is faulty on two actions and the target policies are: ideal, fitted on $\vh\IS$, and fitted on $\vh\WIS$.
  }
  \begin{center}
  \begin{small}
    \begin{tabular}{|c|c|c|c|}
      \hline
      Sample size &	$5000$ &		$10000$ &		$20000$ \\
      \hline
      \ac{ESLB} & \textbf{1.000 $\pm$ 0.004} & \textbf{1.000 $\pm$ 0.004} & \textbf{1.000 $\pm$ 0.004} \\
      $\lambda$-\ac{IS}  & 0.710 $\pm$ 0.443 & 0.837 $\pm$ 0.356 & 0.900 $\pm$ 0.238 \\
      $\lambda$-DR  & $-\infty$ & 0.837 $\pm$ 0.356 & 0.941 $\pm$ 0.140 \\
      Cheb-\ac{WIS}  & $-\infty$ &  $-\infty$ &  $-\infty$ \\
      \hline
      DR  & 0.896 $\pm$ 0.187 & 0.871 $\pm$ 0.279 & 0.951 $\pm$ 0.122 \\
      Emp.Lik.~\citep{karampatziakis2019empirical}  & 0.844 $\pm$ 0.312 & 0.819 $\pm$ 0.354 & 0.883 $\pm$ 0.293 \\
      \hline
      Best policy on the test set  & 1.000 $\pm$ 0.004 & 1.000 $\pm$ 0.004 & 1.000 $\pm$ 0.004 \\
      \hline
    \end{tabular}
    \end{small}
\end{center}
\label{tab:synth1}
\end{table*}

  \begin{table*}[!t]
    \centering
    \caption{
      Average test rewards on a real benchmark.
      Here the behaviour policy is faulty on two actions and the target policies are: ideal, fitted on $\vh\IS$, and fitted on $\vh\WIS$.
      }
    \resizebox{\textwidth}{!}{
      \begin{tabular}{|c|c|c|c|c|c|c|c|c|}
        \hline
         Name &    Yeast & PageBlok & OptDigits & SatImage & isolet & PenDigits & Letter & kropt \\

         Size &   1484 & 5473 & 5620 & 6435 & 7797 & 10992 & 20000 & 28056  \\
         \hline
         \ac{ESLB}   &  \textbf{0.90 $\pm$ 0.27} &  \textbf{0.91 $\pm$ 0.27} & \textbf{0.91 $\pm$ 0.26} & \textbf{0.91 $\pm$ 0.26} & \textbf{0.90 $\pm$ 0.27} & \textbf{0.91 $\pm$ 0.27} & \textbf{0.91 $\pm$ 0.27} & \textbf{0.91 $\pm$ 0.27}\\
         $\lambda$-\ac{IS}  &  \textbf{0.91 $\pm$ 0.26 }   & \textbf{0.91 $\pm$ 0.27} &  0.72 $\pm$ 0.40 & 0.70 $\pm$ 0.39 & 0.75 $\pm$ 0.40 &\textbf{0.9 $\pm$ 0.27} & \textbf{0.90 $\pm$ 0.27} & \textbf{0.90 $\pm$ 0.27}\\
         $\lambda$-DR & $-\infty$ &	\textbf{0.91 $\pm$ 0.27} &	$-\infty$ &	$-\infty$ & \textbf{0.90 $\pm$ 0.27} &	\textbf{0.91 $\pm$ 0.26} & \textbf{0.91 $\pm$ 0.27} & \textbf{0.91 $\pm$ 0.27} \\
                    Cheb-\ac{WIS}  &	  $-\infty$ &	  $-\infty$ &     $-\infty$ &     $-\infty$ &   $-\infty$ & $-\infty$ & \textbf{0.90 $\pm$ 0.27} & $-\infty$\\
                    \hline
        DR  &   0.52 $\pm$ 0.31 & 0.75 $\pm$ 0.36 & 0.68 $\pm$ 0.32 & 0.62 $\pm$ 0.39 & 0.21 $\pm$ 0.29 & 0.79 $\pm$ 0.31 & 0.63 $\pm$ 0.28 & \textbf{0.91 $\pm$ 0.27} \\
        Emp.Lik.~\citep{karampatziakis2019empirical} &  0.31 $\pm$ 0.32 &	0.66 $\pm$ 0.40 & 	0.28 $\pm$ 0.35 & 0.63 $\pm$ 0.40 & 0.21 $\pm$ 0.29	& 0.54 $\pm$ 0.42 & 0.24 $\pm$ 0.33 & 0.71 $\pm$ 0.29\\
        \hline
    \end{tabular}
    }
    \label{tab:real_data_results}
  \end{table*}

Our experiments aim to verify two hypotheses\footnote{Our code is available at \url{https://github.com/deepmind/offpolicy_selection_eslb}.}: (i) the \acf{ESLB} of \cref{thm:WIS_contextual_sim_lb_value} is empirically tighter than its main competitors (discussed in~\cref{sec:baselines}), which is assessed through the value-gap and experiments for the best policy selection problem; (ii) best policy selection based on confidence bounds is superior to selection using just the bare estimators.
Henceforth, we will be concerned mainly with comparisons between confidence bounds.
Therefore, most estimators which do not come with practically computable confidence bounds are outside of the scope of the following experiments.
Our experimental process is inspired by previous work on off-policy evaluation \citep{dudik2011doubly, dudik2014doubly, farajtabar2018more}.

\subsection{Policies and datasets}
\label{sec:summary_exp_setting}
We summarize our experimental setup here; all details can be found in \cref{sec:exp_details}.
We consider a contextual bandit problem such that for every context $\bx$ there is a single action with reward $1$, denoted by $\rho(\bx) \in [K]$, and the reward of all other actions is $0$. This setup is closely related to multiclass classification problems: treating feature vectors as contexts (arriving sequentially) and the predicted label as the action, and defining the reward to be 1 for a correct prediction and 0 otherwise, we arrive at the above bandit problem; this construction has been used in off-policy evaluation (see, e.g., \citep{bietti2018contextual}).

Throughout we consider Gibbs policies: we define an ideal Gibbs policy as
$\pi^{\text{ideal}}(y \mid \bx) \propto e^{\frac{1}{\tau} \mathbb{I}\{y=\rho(\bx)\}}$, where $\mathbb{I}$ denotes the indicator function\footnote{For an event $E$, $\mathbb{I}\{E\}=1$ if $E$ holds, and $0$ otherwise.}
and $\tau > 0$ is a temperature parameter.
The smaller $\tau$ is, the more peaky is the distribution on the predicted label.
To create mismatching policies, we consider a \emph{faulty} policy type for which the peak is shifted to another, wrong action for a set of faulty actions $F \subset [K]$ (i.e., if $\rho(\bx) \in F$, the peak is shifted by $1$ cyclically).
In the following we consider faulty behavior policies, while one among the target policies is \emph{ideal}.

Motivated by the large body of literature on off-policy learning \citep{swaminathan2015batch,swaminathan2015counterfactual,joachims2018deep}, which considers the problem of directly learning a policy from logged bandit feedback, we also consider trained target policies: the policies have a parametric form $\pi^{\bhTheta\subIS}(k|\bx) \propto e^{\frac{1}{\tau} (\bhtheta\subIS)_k\tp \bx}$ and  $\pi^{\bhTheta\subWIS}(k|\bx) \propto e^{\frac{1}{\tau} (\bhtheta\subWIS)_k\tp \bx}$, and their parameters are learned by respectively maximizing the empirical values $\vh\IS$ and $\vh\WIS$ (through gradient descent), to imitate parameter fitting w.r.t.\ these estimators (see \cref{sec:exp_details:learnt_policies} for details).

Some of our experiments require 
a precise control of the distribution of the contexts, as well as of the sample size. To accomplish this, we generate
\emph{synthetic datasets} from an underlying multiclass classification problem
through the scikit-learn function \verb+make_classification()+.

Finally,
8 \emph{real} multiclass classification datasets are chosen from OpenML \citep{dua2017openml} (see Table~\ref{tab:datasets} in \cref{sec:exp_details})
with classification tasks of various sizes, dimensions and class imbalance.

\subsection{Empirical tightness analysis: comparison of existing bounds}

\ac{ESLB} of \cref{thm:WIS_contextual_sim_lb_value}, Chebyshev, $\lambda$-\ac{DR} and $\lambda$-\ac{IS} take as input a parameter $\delta \in (0,1)$ that controls the theoretical error probability of the obtained lower bound (the coverage probability is $1-\delta$). However, there is usually a gap between this theoretical value and the actual empirical coverage obtained in practice.
We fix a synthetic problem  with size $n=10^4$, $\tau_{b}=0.3$ (Gibbs behavior policy, with two faulty actions) and $\tau_t=0.3$ (Gibbs target policy, with one different faulty action).
As an indication of the difficulty of the problem, the effective sample size (see Sec.~\ref{sec:setting}) here is 
$n_{\text{eff}}=655$, which is an order of magnitude smaller than $n$ (a moderate policy mismatch), but should allow a reasonable estimation.

For each value of $\delta$, we repeated the same experiments 100 times, regenerating a new but identically distributed logged dataset and computing the estimates.
The empirical distribution of the lower bound $\hat v_n(\pi)$
and the width $v(\pi)-\hat v_n(\pi)$ at  $\delta=0.05$ (the $\delta$-value used in the experiments) are shown in Figure~\ref{fig:coverage-analysis}.
We can observe that all lower bounds have a positive distance to the true value \emph{for any } $\delta \in (0,1)$.
This means that none of the lower bounds is tight at this sample size. Nonetheless, \ac{ESLB} is considerably tighter for low error probabilities ($\delta \leq 0.1$). It is expected that this tightness is a key ingredient to make more accurate decisions.

Nonetheless, in all 100 repetitions, the lower bound is \emph{always} below the true value, meaning that the empirical coverage of the estimators is $1>1-\delta$ for all $\delta$. This is not the case for \ac{EL}, as shown in Figure~\ref{fig:EL-tests}. 
This simulation highlights two interesting facts about \ac{EL}. We can see that the returned lower bound is always very close to the true value, which should be a perfect property for our selection problem. But unfortunately, the estimated lower bound also suffers from a quite large variance, resulting an empirical coverage below $1-\delta$ when $\delta$ is small.
This is likely the reason for the inferior performance of \ac{EL} on our real data experiments, presented in the next section. 

 \begin{figure}
     \centering
     \includegraphics[width=0.4\textwidth]{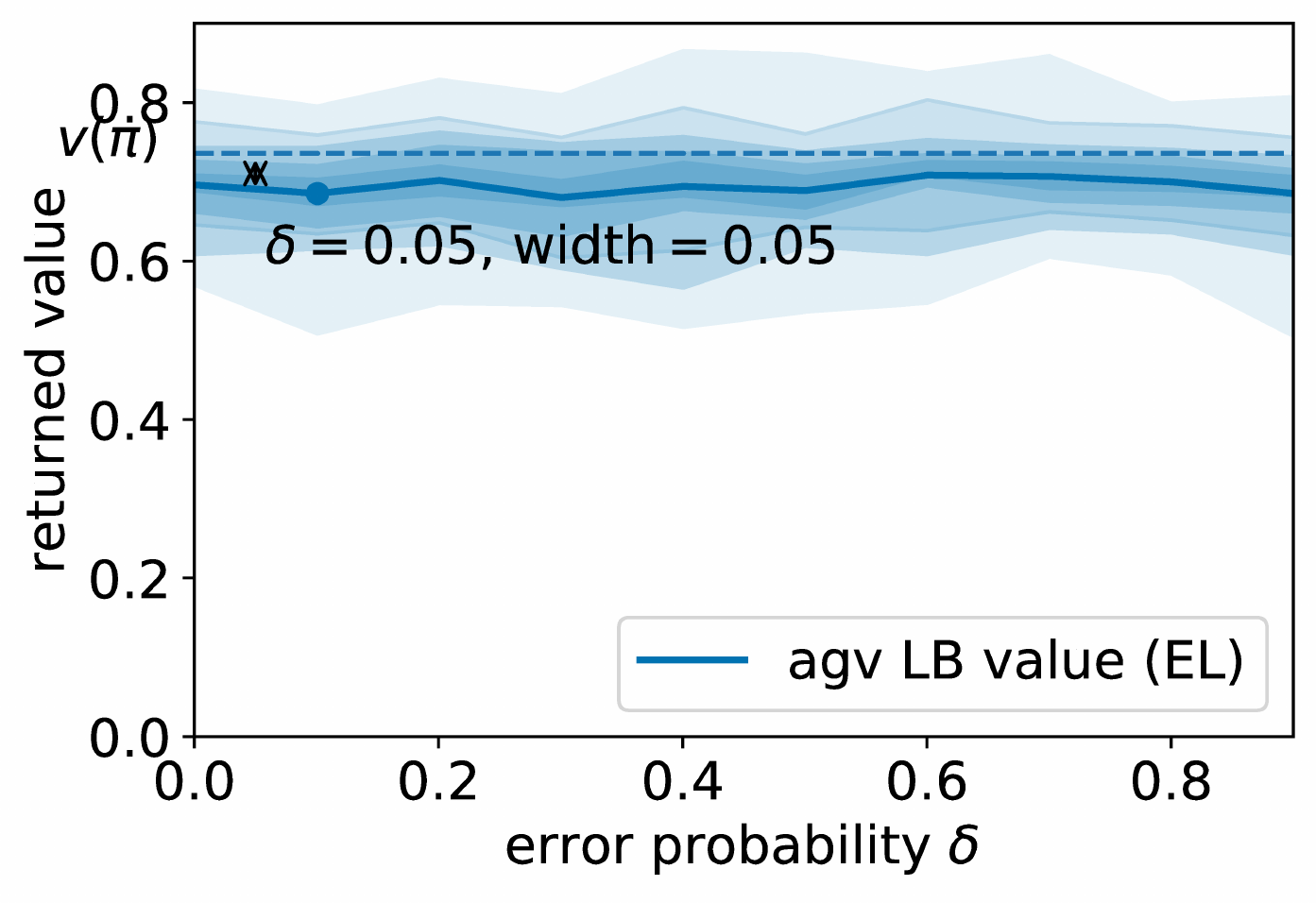}
     \caption{Empirical tightness of the EL lower bound estimator. The returned value is on average very close to the true value, but the empirical coverage is below $1-\delta$ for small values of $\delta$. 
    }
     \label{fig:EL-tests}
 \end{figure}

\subsection{Best-policy selection}
We evaluate all estimators on the best-policy selection
problem (see Figure~\ref{fig:evaluator}). For all experiments, we use a behavior policy with two faulty actions and temperature $\tau = 0.2$. The set of candidate target policies is $\pi^{\text{ideal}}, \pi^{\bhTheta\subIS}, \pi^{\bhTheta\subWIS}$ with temperature $\tau=0.2$ for synthetic and (almost deterministic) $\tau=0.1$ for real datasets.
The performance of a selected policy is its average reward collected on a separate test sample over $10$ independent trials ($5\cdot 10^4$ examples in synthetic case).
In all cases, we set the error probability $\delta=0.01$ except for OptDigits and SatImage where it is $\delta=0.05$.%
\footnote{For those datasets, all confidence intervals were vacuous with $\delta=0.01$ so we adjusted it to obtain exploitable results.}
Since the best-policy selection problem relies on the comparison of $N=3$ confidence bounds, by application of the union bound, the final result holds with a lower probability, i.e.\ $\delta$ is replaced by $\delta/N$.

\cref{tab:synth1} presents the results on a synthetic benchmark.
\ac{ESLB} perfectly returns the best policy on all trials while other estimators fail at least once.

Results on real data, summarized in \cref{tab:real_data_results}, show that \ac{ESLB} also achieves the best average performance in all cases, but other confidence-based methods turn out to be reasonable alternatives --- especially for large samples. This is in contrast with the pure estimators \ac{DR} and \ac{EL} \citep{karampatziakis2019empirical} that consistently make selection mistakes and show significantly lower performance on average on all datasets. This confirms our hypothesis that selection must be performed by a confidence-based scoring function.

The detailed decomposition of bounds in both \cref{tab:synth1,tab:real_data_results} into constituent terms (concentration, bias, etc.) can be found in~\cref{sec:decomp}.

\subsection{Discussion}

\paragraph{Confidence bound for the off-policy selection.}
In this paper we demonstrated that the high probability lower bound on the value can be used effectively to select the best policy.
We considered only a finite set of $N$ target policies, however bounds on the value can be combined using a \emph{union bound} to provide a bound on the value of the selected policy.
In such case \cref{thm:WIS_contextual_sim_lb_value} would hold with a slight modification, where we would pay a $x + \ln(N)$ term in place of $x$.
In other words, for $\pi^{\star} \in \max_{\pi \in \cbr{\pi_1, \ldots, \pi_N}} \ESLB_{x + \ln(N)}(\pi)$
where $\ESLB_x(\pi)$ denotes a right hand side of the bound in \cref{thm:WIS_contextual_sim_lb_value},
with probability at least $1-2 e^{-x}$ we have $v(\pi^{\star}) \geq \ESLB_{x + \ln(N)}(\pi^{\star})$.

Naturally, one might wonder how to extend the above to the uncountable class of policies where one could maximize the lower bound over policies: Such optimization is known as the \emph{off-policy policy optimization}.
Several works explored such possibility for \ac{WIS} estimator by showing bounds in the \ac{PAC} framework~\citep{swaminathan2015self,athey2021policy}.
\ac{PAC} bounds are known to be conservative in general: In our case they would hold with respect to the \emph{worst} policy in a given policy class.
The concentration inequality we build upon in this work (\cref{thm:concentration}) is also available in a so-called \ac{PAC}-Bayes formulation~\citep{kuzborskij2019efron}, where the inequality holds for all probability measures over a given policy space simultaneously.
Such \ac{PAC}-Bayes formulation offers an alternative to the \ac{PAC} formulation, while it was shown to be less conservative numerically~\citep{dziugaite2017computing}.

\textbf{Upper bound on the value. } Our concentration inequality is symmetrical and should in principle allow to obtain an upper-bound on the value as well. Note however that for our use case we aim at finding the policy with highest value and so for that task the lower bound suffices. Nonetheless, upper-bounds could be computed to get a sense of the \emph{width} of the confidence interval and thus of its tightness.

\section{Conclusions and future work}

We derived the first high-probability truncation-free finite-sample confidence bound on the value of a contextual bandit policy that employs \ac{WIS} estimator, which turned into a practical and principled off-policy selection method.
The sharpness of our bound is due to a careful handling of the empirical variance of \ac{WIS} estimator, but this desirable property comes at the cost of an increased computation complexity.
Indeed, Monte-Carlo simulations are needed to obtain the key terms in the bound.
These efforts allow us to achieve state-of-the-art performance on a variety of contextual bandit tasks.
Nevertheless, we see several future directions for our work.

\paragraph{Off-policy learning. } We have demonstrated that \ac{ESLB} can be used to improve the behavior policy by choosing a better target policy provided there is one in the finite set of candidates.
In principle, we can obtain even better improvement by
selecting the policy from an uncountable class by maximizing the \ac{ESLB} --or a similar-- e.g. as in \cite{swaminathan2015counterfactual}.

\paragraph{ESLB for MDPs. } Extending our analysis to stateful Markov Decision Processes is a challenging but promising direction. Indeed, off-policy evaluation and learning is a major topic of research in reinforcement learning. In that more general setting, for each learning episode, a policy generates not only a reward but an entire sequence of states, actions and rewards, which eventually characterize its value. We believe that similar techniques as those applied here may allow to obtain high-probability bounds on the policy value, at least in finite-horizon MDPs.

\subsubsection*{Acknowledgements}
We are grateful to Tor Lattimore for many insightful comments.

%

\newpage

\appendix

\onecolumn

\hsize\textwidth
  \linewidth\hsize \toptitlebar {\centering
  {\Large\bfseries Confident Off-Policy Evaluation and Selection through Self-Normalized Importance Weighting: Supplementary Material \par}}
\bottomtitlebar

\aistatsauthor{Ilja Kuzborskij \And Claire Vernade \And Andr\'as Gy\"orgy \And Csaba Szepesv\'ari}~\\
\aistatsaddress{ DeepMind }

\section{Revision of the Bias Term (Errata)}
\label{sec:errata}
In the original version of the paper the bias of estimator $\wh{v}\WIS(\pi)$ was controlled by relying on a relationship between $\E\br{\wh{v}\WIS(\pi) \mid X_1^n}$ and $v(\pi)$:
\begin{align}
  \label{eq:bias_harris_problematic}
  \E\br{\wh{v}\WIS(\pi) \mid X_1^n}
  &= \E\br{\frac{\sum_{k=1}^n W_k R_k}{\sum_{k=1}^n W_k} \bmid X_1^n} \\
  &\stackrel{(*)}{\leq} \underbrace{\E\br{ \frac{1}{\sum_{k=1}^n W_k} \bmid X_1^n}}_{(a)}\left(\sum_{k=1}^n v(\pi \pmid X_k)\right)~. \nonumber
\end{align}
This was done through the use of the Harris' inequality:
\begin{theorem}[Harris' inequality {\cite[Theorem~2.15]{boucheron2013concentration}}]
  \label{thm:harris}
  Let $f : \reals^n \to \reals$ be a non-increasing and $g : \reals^n \to \reals$ be a non-decreasing function.
  Then for real-valued random variables $(X_1, \ldots, X_n)$ independent from each other, we have
  $
    \E[f(X_1, \ldots, X_n) g(X_1, \ldots, X_n)] \leq \E[f(X_1, \ldots, X_n)] \E[g(X_1, \ldots, X_n)]~.
  $
\end{theorem}
In particular, the original proof notes that $(w_1, \ldots, w_n) \to \sum_i w_i R_i$ is non-decreasing a.s.\ and $(w_1, \ldots w_n) \to (\sum_i w_i)^{-1}$ is non-increasing (since the $R_i$ are non-negative), and
 applies Harris' inequality to get \cref{eq:bias_harris_problematic}.
Combining \cref{eq:bias_harris_problematic} with Hoeffding's inequality the proof then got to
\begin{align*}
  \lefteqn{v(\pi) - \E\br{\wh{v}\WIS(\pi) \mid X_1^n}} \\
  &\stackrel{(*)}{\geq}
    v(\pi) - \E\br{\frac{1}{\sum_{k=1}^n W_k} \bmid X_1^n} \sum_{k=1}^n v(\pi \pmid X_k)\\
      &\geq
    v(\pi) \pr{1 - \E\br{ \frac{n}{\sum_{k=1}^n W_k} \bmid X_1^n}} - \E\br{\frac{\sqrt{n x / 2}}{\sum_{k=1}^n W_k} \bmid X_1^n}
\end{align*}
which holds with probability at least $1-e^{-x}, x > 0$.
Unfortunately, the step $(*)$ does not hold: we do not know if the Harris inequality applies because $(w_1, \ldots, w_n) \to \sum_i w_i R_i$ is not necessarily non-decreasing.
The problem is that the rewards $R_1, \ldots, R_n$ \emph{depend} on the importance weights through the actions.
The proof of $(*)$ remains an open problem.

\paragraph{Fix.}
Here we present a control of the bias of a very similar form while avoiding Harris' inequality.
The price we pay is an assumption that the reward function if `one-hot' --- note that this is the case in all our experiments and in general in all problems akin to classification, i.e. where only one action is the correct label for a given context.
\begin{nameddef}[\cref{lem:one_hot_id}]
  Suppose that reward function
  is `one-hot', that is, for any context $x$, $r(\astar \mid x)=1$ for some unique action $\astar$ and $r(a \mid x)=0$ for all $a \neq \astar$.
  Denote $\wstar_i = \frac{\pi(\astar \mid X_i)}{\pi_b(\astar \mid X_i)}$.
  Then,
\begin{align*}
  \E\br{\frac{\sum_{i=1}^n W_i R_i}{\sum_{j=1}^n W_j} \bmid X_1^n}
  =
  \sum_{i=1}^n v(\pi \pmid X_i) \E\br{\frac{1}{w^{\star}_i + \sum_{j \neq i} W_j} \bmid X_1^n}~.
\end{align*}
\end{nameddef}
\begin{proof}
  Proof is given in \cref{sec:proof_from_the_main_text}.
\end{proof}
Clearly, since the optimal action is unknown, taking $w^{\star}_i \geq 0$, \cref{lem:one_hot_id} implies inequality
\begin{align*}
  \E\br{\wh{v}\WIS(\pi) \mid X_1^n} \leq
  \underbrace{\E\br{\frac{1}{\min_i \sum_{j \neq i} W_j} \bmid X_1^n}}_{(b)} \pr{\sum_{i=1}^n v(\pi \pmid X_i)}~.
\end{align*}
We observe that $(*)$ only differs from the above by replacing term $(a)$ with term $(b)$.
Clearly $(b)$ is a larger quantity, and in the following we numerically assess the extent of how much we loose.
We also note that none of our numerical results are affected: as we see from tables below the difference is minor and even the new bias $(b)$ remains close to $1$ (recall that it appears multiplicatively w.r.t.\ the concentration term).
\begin{table}[H]
  \centering
  \begin{tabular}{|l|c|c|c|}
    \hline
    Sample size & 5000 & 10000 & 20000\\
    \hline
    New bias: Gibbs-fitted-IW & $0.9590 \pm 0.0009$ & $0.9792 \pm 0.0006$ & $0.9895 \pm 0.0004$\\
    New bias: Gibbs-fitted-SN & $0.9430 \pm 0.0021$ & $0.9720 \pm 0.0012$ & $0.9867 \pm 0.0005$\\
    New bias: Ideal & $0.9575 \pm 0.0009$ & $0.9784 \pm 0.0006$ & $0.9892 \pm 0.0004$\\
    \hline
    Old (Harris' ineq.) bias: Gibbs-fitted-IW & $0.9899 \pm 0.0009$ & $0.9945 \pm 0.0006$ & $0.9971 \pm 0.0004$\\
    Old (Harris' ineq.) bias: Gibbs-fitted-SN & $0.9743 \pm 0.0021$ & $0.9874 \pm 0.0012$ & $0.9944 \pm 0.0005$\\
    Old (Harris' ineq.) bias: Ideal & $0.9884 \pm 0.0009$ & $0.9937 \pm 0.0006$ & $0.9968 \pm 0.0004$\\
    \hline
  \end{tabular}
  \caption{Comparison of new (b) and old (a) bias terms on a synthetic benchmark.}
\end{table}

\begin{table}[H]
  \centering
\resizebox{\textwidth}{!}{
  \begin{tabular}{|l|c|c|c|c|c|c|c|c|}
    \hline
    Bias type: policy &    Yeast & PageBlok & OptDigits & SatImage & isolet & PenDigits & Letter & kropt \\
    Sample size &   1484 & 5473 & 5620 & 6435 & 7797 & 10992 & 20000 & 28056  \\
    \hline
    New bias: Gibbs-fitted-IW & $0.7053 \pm 0.0296$ & $0.9290 \pm 0.0033$ & $0.9397 \pm 0.0039$ & $0.9393 \pm 0.0057$ & $0.9539 \pm 0.0023$ & $0.9618 \pm 0.0026$ & $0.9707 \pm 0.0009$ & $0.9784 \pm 0.0010$\\
    New bias: Gibbs-fitted-SN & $0.5635 \pm 0.0415$ & $0.9220 \pm 0.0042$ & $0.8925 \pm 0.0061$ & $0.9185 \pm 0.0063$ & $0.9161 \pm 0.0020$ & $0.9543 \pm 0.0028$ & $0.9689 \pm 0.0008$ & $0.9779 \pm 0.0011$\\
    New bias: Ideal & $0.9791 \pm 0.0001$ & $0.9594 \pm 0.0004$ & $0.9343 \pm 0.0011$ & $0.9371 \pm 0.0009$ & $0.9601 \pm 0.0005$ & $0.9656 \pm 0.0007$ & $0.9818 \pm 0.0004$ & $0.9864 \pm 0.0003$\\
    \hline
    Old (Harris' ineq.) bias: Gibbs-fitted-IW & $0.8964 \pm 0.0168$ & $0.9849 \pm 0.0027$ & $0.9920 \pm 0.0022$ & $0.9871 \pm 0.0050$ & $0.9950 \pm 0.0010$ & $0.9908 \pm 0.0026$ & $0.9882 \pm 0.0009$ & $0.9903 \pm 0.0010$\\
    Old (Harris' ineq.) bias: Gibbs-fitted-SN & $0.8152 \pm 0.0245$ & $0.9788 \pm 0.0039$ & $0.9519 \pm 0.0059$ & $0.9680 \pm 0.0059$ & $0.9626 \pm 0.0020$ & $0.9835 \pm 0.0027$ & $0.9865 \pm 0.0008$ & $0.9898 \pm 0.0011$\\
    Old (Harris' ineq.) bias: Ideal & $0.9963 \pm 0.0006$ & $0.9964 \pm 0.0006$ & $0.9889 \pm 0.0012$ & $0.9854 \pm 0.0009$ & $0.9964 \pm 0.0006$ & $0.9944 \pm 0.0007$ & $0.9988 \pm 0.0003$ & $0.9982 \pm 0.0003$\\
    \hline
  \end{tabular}
}
\caption{Comparison of new (b) and old (a) bias terms on a UCI datasets.}
\end{table}

\section{More on the stopping criteria for $\tilde{V}\WIS_t$, $\tilde{U}\WIS_t$, and $\tilde{B}_t$ in \cref{alg:v_WIS2}}
\label{sec:appendix:stopping}

As discussed in \cref{sec:sim2}, we control the simulation error introduced by the output of \cref{alg:v_WIS2} by applying a stopping criterion based on the empirical Bernstein's inequality (\cref{thm:empirical_bernstein}).
In particular, for a user specified precision $\ve > 0$, the estimation of $V\WIS$ is stopped when
\begin{equation*}
  \ve \geq \sqrt{\frac{2 \widehat{\Var}(\tilde{V}\WIS_t)}{t}} + \frac{7}{3} \cdot \frac{2 x}{t-1}
\end{equation*}
is satisfied.
Suppose that the simulation has stopped after $T_{\ve}$ iterations.
Then, the above guarantees w.p.\ at least $1-e^{-x}, x > 0$ that $|V\WIS - \tilde{V}\WIS_{T_{\ve}}| \leq \ve$.
We note that this comes by a direct application of \cref{thm:empirical_bernstein} where the range $C=2$, since $V\WIS \leq 2$ a.s.

Similarly, we have a stopping criterion for $U\WIS$, that is we stop when
\begin{equation*}
  \ve \geq \sqrt{\frac{2 \widehat{\Var}(\tilde{U}\WIS_t)}{t}} + \frac{7}{3} \cdot \frac{2 x}{t-1}
\end{equation*}
is satisfied.
This gives w.h.p\ $|U\WIS - \tilde{U}\WIS_{T_{\ve}}| \leq \ve$.

In case of $\tilde{B}_T$, we control its simulation error indirectly through controlling an error $|Z^{\mathrm{inv}}_{T_{\ve}} - 1/Z| \leq \ve$, i.e.\ stopping when
\begin{equation*}
  \ve \geq \sqrt{\frac{2 \widehat{\Var}(Z^{\mathrm{inv}}_t)}{t}} + \frac{7}{3} \cdot \frac{M x}{t-1}~,
\end{equation*}
is satisfied, where $M = 1 / \sum_i \min_{a \in [K]} \frac{\pi(a|X_i)}{\pi_b(a|X_i)}$ (note that $1/Z \leq M$ a.s.\ for fixed $X_1^n$).
The reason for this becomes clear by observing a simple lower bound on $B$:
\[
  B = \min\pr{1, \frac{1}{\E\br{\frac{n}{Z} \bmid X_1^n}}}
  \geq
  \min\pr{1, \frac{1}{\E\br{n (Z^{\mathrm{inv}}_{T_{\ve}} + \ve) \bmid X_1^n}}}~.
\]

Finally, we note that convergence of $\tilde{V}\WIS_t$, $\tilde{U}\WIS_t$, and $\tilde{B}_t$ might take different number of steps and in practice one would split~\cref{alg:v_WIS2} into separate subroutines for estimation of respective quantities with different stopping criteria.
As mentioned before the sample variance can be easily computed online, for instance by using Welford's method.

\section{Additional proofs}
\label{sec:additional_proofs}

\subsection{Proofs from \cref{sec:proofs}}
\label{sec:proof_from_the_main_text}

\begin{nameddef}[\cref{lem:one_hot_id}]
  Suppose that reward function
  is `one-hot', that is, for any context $x$, $r(x, \astar)=1$ for some unique action $\astar$ and $r(x, a)=0$ for all $a \neq \astar$.
  Denote $\wstar_i = \frac{\pi(\astar \mid X_i)}{\pi_b(\astar \mid X_i)}$.
  Then,
\begin{align*}
  \E\br{\wh{v}\WIS(\pi) \bmid X_1^n}
  =
  \sum_{i=1}^n v(\pi \pmid X_i) \E\br{\frac{1}{w^{\star}_i + Z\deli} \bmid X_1^n}~.
\end{align*}
\end{nameddef}
\begin{proof}[Proof of \cref{lem:one_hot_id}]
  Consider a single summand of $\E[\vh\WIS(\pi) \mid X_1^n]$: we expand the expectation over $\pi_b$ and simplify the $i$-th term of the sum by noticing that for one-hot reward functions, $v(\pi \mid X_i) = \pi(\astar \mid X_i)$.
  \begin{align*}
  \E\br{\frac{W_i R_i}{W_1 + \dots + W_n} \bmid X_1^n}
  &= \sum_{a_1, \ldots, a_n}  \frac{\frac{\pi(a_i \mid X_i)}{\pi_b(a_i \mid X_i)} \cdot r(a_i \mid X_i)}{\frac{\pi(a_1 \mid X_1)}{\pi_b(a_1 \mid X_1)} + \dots + \frac{\pi(a_n \mid X_n)}{\pi_b(a_n \mid X_n)}} \cdot \pi_b(a_1 \mid X_1) \dots \pi_b(a_n \mid X_n)\\
    &= \sum_{a_1, \ldots, a_{i-1}, a_{i+1}, \dots, a_n}  \frac{\frac{\pi(\astar \mid X_i)}{\pi_b(\astar \mid X_i)}}{\frac{\pi(\astar \mid X_i)}{\pi_b(\astar \mid X_i)} + \sum_{j\neq i}\frac{\pi(a_j \mid X_j)}{\pi_b(a_j \mid X_j)}} \cdot \pi_b(\astar \mid X_i) \prod_{j \neq i} \pi_b(a_j \mid X_j)\\
    &= v(\pi \pmid X_i) \E\br{\frac{1}{\wstar_i + \sum_{j \neq i} W_j} \bmid X_1^n}
  \end{align*}
  Now, summing over $i \in [n]$ we get the statement.
\end{proof}

Note that we obtain an equality, and we can further upper-bound the rightmost term by $\E\br{\frac{1}{\min_{i\in[n]} \wstar_i + \sum_{j \neq i} W_j} \bmid X_1^n}$.

To prove \cref{prop:V_WIS_loo} we will need the following statement:
\begin{prop}
    \label{prop:V_loo}
    Let $S=((W_i,R_i))_{i=1}^n$ be independent random variables distributed according to some probability measure on $\sY_1\times \dots \times \sY_n$, let $f(S) = \frac{\sum_{i=1}^n W_i R_i}{\sum_{i=1}^n W_i}$, and $f_k(S\repk) = \frac{\sum_{i\neq k} W_i R_i}{\sum_{i \neq k}^n W_i}$.
    Let $E_k = R_k - f_k(S\delk)$.
    Then for all $k \in [n]$,
    \begin{align*}
      f(S) - f_k(S\delk) = \frac{W_k E_k}{\sum_{i=1}^n W_i}~.
    \end{align*}
  \end{prop}

\begin{prop}
  \label{prop:V_WIS_loo}
  Let $f(S) = \frac{\sum_{i=1}^n W_i R_i}{\sum_{i=1}^n W_i}$.
  Then,
  \[
    \sum_{k=1}^n \E\br{(f(S) - f(S\repk))^2 \bmid W_1^k, X_1^n}
    \leq V\WIS
    = \sum_{k=1}^n \E\br{\pr{\frac{W_k}{\sum_{i=1}^n W_i} + \frac{W_k'}{W_k' + \sum_{i \neq k} W_i}}^2 \bmid W_1^k, X_1^n}~.
  \]
\end{prop}
\begin{proof}
    Denote
    \begin{align*}
      \tilde{W}_k = \frac{W_k}{\sum_{i=1}^n W_i}~, \qquad \tilde{U}_k = \frac{W_k'}{W_k' + \sum_{i \neq k} W_i} \qquad k \in [n]~.
    \end{align*}
    By~\cref{prop:V_loo}
    \begin{align*}
      f(S) - f(S\repk)
      &=
        f(S) - f_k(S\delk) + f_k(S\delk) - f(S\repk)\\
      &=
        \frac{W_k E_k}{\sum_{i=1}^n W_i} - \frac{W_k' E_k'}{W'_k + \sum_{i \neq k} W_i}
        =
        \tilde{W}_k E_k - \tilde{U}_k E_k'
    \end{align*}
    where $E_k' = R_k' - f_k(S\delk)$.
    Taking square on both sides gives
    \begin{align*}
      \pr{f(S) - f(S\repk)}^2
      &=
        \tilde{W}_k^2 E_k^2 + \tilde{U}_k^2 (E_k')^2 - 2 \tilde{W}_k \tilde{U}_k E_k E_k' \\
      &\leq
        \tilde{W}_k^2 + \tilde{U}_k^2  + 2 \tilde{W}_k \tilde{U}_k \tag{Since $E_k, E_k' \in [-1, 1]$ a.s.} \\
      &=
        \pr{\tilde{W}_k + \tilde{U}_k}^2~.
    \end{align*}
  \end{proof}

\begin{proof}
From simple algebra (see Proposition 2 in \citep{kuzborskij2019efron} discussion), we have
  \[
    f(S) - f_k(S\delk) = \frac{W_k(R_k - f_k(S\delk))}{\sum_{i=1}^n W_i} \leq \frac{W_k}{\sum_{i=1}^n W_i} \qquad k \in [n]~.
  \]
  Then, the desired result follows from an application of \cref{prop:V_loo}, with $f = \vh\WIS$ and $S = \pr{(W_1, R_1), \ldots, (W_n, R_n)}$, given the contexts. 
\end{proof}

\subsection{Polynomial Bounds for Weighted Importance Sampling}

Since the exact calculation of $V\WIS$ could be prohibitive, we use a shortcut to lower bound the denominator $\sum_i W_i$.
The promised lower bound is based on the following (more or less standard) result:
\begin{lemma}
  \label{lem:bernstein_lower_tail}
  Assume that the non-negative random variables $W_1, W_2, \ldots, W_n$ are distributed independently from each other given $\sF_0$.
  Then, for any $t \in [0, \sum_{k=1}^n \E[W_k\mid\sF_0])$,
  \begin{align*}
    \P\pr{\sum_{i=1}^n W_i \leq t \bmid \sF_0} 
	\leq \exp\pr{- \frac{\pr{t - \sum_{k=1}^n \E\br{W_k\mid\sF_0}}^2}{2 \sum_{k=1}^n \E\br{W_k^2\mid\sF_0}}}
  \end{align*}
  and in particular for any $x>0$, with probability at least $1-e^{-x}$,
  \begin{equation}
    \label{eq:b_n}
    \sum_{i=1}^n W_i
    >
    \sum_{k=1}^n \E[W_k\mid\sF_0] - \sqrt{2 x \sum_{k=1}^n \E\br{W_k^2\mid\sF_0}}~.
  \end{equation}
\end{lemma}
\begin{proof}
We drop conditioning on $\sF_0$ to simplify notation.
  Chernoff bound readily gives a bound on the lower tail
  \begin{align*}
    \P\pr{\sum_{i=1}^n X_i \leq t} \leq \inf_{\lambda > 0} e^{\lambda t} \E\br{e^{- \lambda \sum_{i=1}^n X_i}}~.
  \end{align*}
  By independence of $X_i$
  \begin{align*}
    \prod_{i=1}^n \E\br{e^{- \lambda X_i}}
    &\leq
    \prod_{i=1}^n \pr{1 - \lambda \E\br{X_i} + \frac{\lambda^2}{2} \E\br{X_i^2} } \tag{$e^{-x} \leq 1 - x + \frac{1}{2} x^2$ for $x \geq 0$}\\
    &\leq
    e^{- \lambda \sum_{i=1}^n \E\br{X_1} + \frac{\lambda^2}{2} \sum_{i=1}^n \E\br{X_i^2}} \tag{$1+x \leq e^x$ for $x \in \reals$ and i.i.d.\ assumption}
  \end{align*}
  Getting back to the Chernoff bound gives,
  \begin{align*}
    \lambda = \max\cbr{\frac{\sum_{i=1}^n \E\br{X_i} - t}{\sum_{i=1}^n \E\br{X_i^2}}, 0}~.
  \end{align*}
  This proves the first result.
  The second result comes by inverting the bound and solving a quadratic equation.
\end{proof}

\paragraph{\cref{prop:wis_cheb} (restated).}
\emph{With probability at least $1-3 e^{-x}$ for $x > 0$,}
\[
  v(\pi) \geq
  \frac{N_x}{n} \pr{
  \vh\WIS(\pi)
  - \sqrt{\frac{\sum_{k=1}^n \E[W_k^2 | X_k]}{N_x^2} \, e^x}
  }
  - \sqrt{\frac{x}{2 n}}~.
\]
where
\[
  N_x = n - \sqrt{2 x \sum_{k=1}^n \E\br{W_k^2\mid X_1^n}}~.
\]
\begin{proof}[Proof of \cref{prop:wis_cheb}]
The decomposition into the bias and the concentration is as in the proof of~\cref{thm:WIS_contextual_sim_lb_value}, where the concentration of contexts is handled once again through Hoeffding's inequality.
Hence, we'll focus only on the concentration.

Let $Z = \vh\WIS(\pi) - \E[\vh\WIS(\pi)]$.
Chebyshev's inequality gives us:
\begin{align*}
  \P\pr{|Z| \geq \sqrt{t \Var(\vh\WIS(\pi) \mid X_1^n)}} \leq \frac{1}{t} \qquad t > 0~.
\end{align*}
This implies
\begin{align*}
  \abs{\vh\WIS(\pi) - \E[\vh\WIS(\pi)]}
  \leq
  \sqrt{e^x \, \Var(\vh\WIS(\pi) \mid X_1^n)} \tag{w.p. at least $1-e^{-x}$, $x > 0$}\\
  \leq
  \sqrt{\frac{e^x}{N_x^2} \sum_{k=1}^n \E[W_k^2 | X_k]} \tag{w.p. at least $1-2 e^{-x}$ (union bound)}
\end{align*}
where by Efron-Stein's inequality and Proposition~2 of~\cite{kuzborskij2019efron}:
\begin{align*}
\Var(\vh\WIS(\pi) \mid X_1^n)
\leq
\E\br{\sum_{k=1}^n \pr{\vh\WIS_S(\pi) - \vh\WIS_{S\delk}(\pi)}^2 \;\middle|\; X_1^n}
\leq
\E\br{\frac{\sum_{k=1}^n W_k^2}{(\sum_{i=1}^n W_i)^2} \;\middle|\; X_1^n}
\end{align*}
and a lower bound on the sum of weights comes from~\cref{lem:bernstein_lower_tail}.
\end{proof}

\subsection{Confidence Bound for $\lambda$-Corrected Importance Sampling Estimator}
\label{sec:IS_lambda}
Recall the following empirical Bernstein bound given in Theorem~\ref{thm:empirical_bernstein}.
\begin{theorem}[\cite{maurer2009empirical}\footnote{\cite{maurer2009empirical} stated inequality in another direction. However, we can show the one we stated by the symmetry of Bernstein's inequality.}]
  \label{thm:empirical_bernstein}
  Let $Z, Z_1, \ldots, Z_n$ be i.i.d.\ random variables with values in $[0,C]$ and let $\conf > 0$.
  Then with probability at least $1-2e^{-\conf}$,
  \[
    \frac1n \sum_{i=1}^n Z_i - \E[Z] \leq \sqrt{\frac{2 \Var(Z_1, \ldots, Z_n) \conf}{n}} + \frac{7 C \conf}{3 (n-1)}
  \]
  where sample variance is defined as
  \begin{equation}
    \Var(Z_1, \ldots, Z_n) = \frac{1}{n (n-1)} \sum_{1 \leq i < j \leq n} \pr{Z_i - Z_j}^2~.
  \end{equation}
\end{theorem}
The following proposition states a concentration bound for the value when using the $\lambda$-\ac{IS} estimator.
\paragraph{Proposition \ref{prop:IS_lambda_bernstein} (restated).}
\emph{
  For the $\lambda$-\ac{IS} estimator we have with probability at least $1-3e^{-\conf}$, for $\conf > 0$,
  \begin{align*}
    v(\pi)
    &\geq
      \wh{v}\ISl(\pi)
      -\sqrt{\frac{2x}{n} \, \Var(\wh{v}\ISl(\pi) \mid X_1^n)}
      - \frac{7 x}{3 \lambda (n-1)}\\
    &- \frac1n \sum_{k=1}^n \sum_{a \in [K]} \pi(a | X_k) \abs{\frac{\pi_b(a \mid X_k)}{\pi_b(a \mid X_k) + \lambda} - 1}
      - \sqrt{\frac{x}{2 n}}~.
  \end{align*}
  and the variance of the estimator is defined as
  \begin{equation}
    \Var(\wh{v}\ISl(\pi)) = \frac{1}{n (n-1)} \sum_{1 \leq i < j \leq n} \pr{W_i^{\lambda} R_i - W_j^{\lambda} R_j}^2~.
  \end{equation}
  }
\begin{proof}
We start with the decomposition
\begin{align*}
  v(\pi) - \wh{v}\ISl(\pi)
  =
  \underbrace{
  v(\pi) - \E\br{\wh{v}\ISl(\pi) \mid X_1^n}
  }_{\mathrm{Bias}}
  +
  \underbrace{
  \E\br{\wh{v}\ISl(\pi) \mid X_1^n} - \wh{v}\ISl(\pi)
  }_{\mathrm{Concentration}}~.
\end{align*}
Observing that $W_k^{\lambda} \leq 1/\lambda$
The concentration term is bounded by \cref{thm:empirical_bernstein} with $C = 1/\lambda$, that is:
\[
  \E\br{\wh{v}\ISl(\pi) \mid X_1^n} - \wh{v}\ISl(\pi)
  \geq
  \sqrt{\frac{2x}{n} \, \Var(\wh{v}\ISl(\pi) \mid X_1^n)}
  + \frac{7 x}{3 \lambda (n-1)}~.
\]
Now we focus on the bias term which is further decomposed as follows:
\begin{align*}
  v(\pi) - \E\br{\wh{v}\ISl(\pi) \pmid X_1^n}
  =
  v(\pi) - \frac1n \sum_{k=1}^n v(\pi | X_k) + \frac1n \sum_{k=1}^n v(\pi | X_k) - \E\br{\wh{v}\ISl(\pi) \pmid X_1^n}
\end{align*}
Since $(v(\pi \pmid X_k))_{k \in [n]}$ are independent and they take values in the range $[0,1]$, 
by Hoeffding's inequality we have w.p.\ at least $1-e^{-x}, x \geq 0$ that
\begin{align*}
  \frac1n \sum_{k=1}^n v(\pi \pmid X_k) - v(\pi) \leq \sqrt{\frac{x}{2 n}}~.
\end{align*}
Finally,
\begin{align*}
  \E\br{\wh{v}\ISl(\pi) \bmid X_1^n} - \frac1n \sum_{k=1}^n v(\pi \pmid X_k)
  &=
    \frac1n \sum_{k=1}^n \E\br{ \pr{W_k^{\lambda} - W_k} R_k \bmid X_k}\\
  &=
    \frac1n \sum_{k=1}^n \E\br{ \pr{\frac{\pi(A_k \mid X_k)}{\pi_b(A_k \mid X_k) + \lambda} - \frac{\pi(A_k \mid X_k)}{\pi_b(A_k \mid X_k)}} R_k \bmid X_k}\\
  &=
    \frac1n \sum_{k=1}^n \sum_{a \in [K]} \pi(a | X_k) \pr{\frac{\pi_b(a \mid X_k)}{\pi_b(a \mid X_k) + \lambda} - 1} r(X_k, a)\\
  &\leq
    \frac1n \sum_{k=1}^n \sum_{a \in [K]} \pi(a | X_k) \abs{\frac{\pi_b(a \mid X_k)}{\pi_b(a \mid X_k) + \lambda} - 1}~.
\end{align*}
Putting all together and applying a union bound we get the statement w.p.\ at least $1-3 e^{-x}$.
\end{proof}

\subsection{Confidence Bound for $\lambda$-Corrected \acl{DR} Estimator}
\label{sec:DR_lambda}
Doubly-Robust (\ac{DR}) estimators were introduced in the machine learning literature for off-policy evaluation by \cite{dudik2011doubly}, and refined in \cite{farajtabar2018more,su2019doubly}. They combine a direct model estimator and \ac{IS}, finding a compromise that should behave like \ac{IS} with a reduced variance. To compute $\vh\DR$, a reward estimator $\eta : \sX \times [K] \to [0,1]$ must be learned on a subset of the logged dataset. Then, 
\[
\vh\DR(\pi) = \hat{V}_{\eta}(\pi) + \frac{1}{n} \sum_{i=1}^n W_i(R_i - \eta(X_i,A_i)),
\]
where $\hat{V}_{\eta}(\pi) = (1/n) \sum_{i=1}^n \sum_{a \in [K]}\pi(a|X_i)\eta(X_i,a)$ is the expected reward of $\pi$ given $\eta$.
Now we prove a very similar bound for the $\lambda$-Corrected \acl{DR} estimator.
\paragraph{Proposition \ref{prop:DR_lambda_bernstein} (restated).}
\emph{For the $\lambda$-\ac{DR} estimator defined w.r.t.\ a fixed $\eta : \sX \times [K] \to [0,1]$ we have with probability at least $1-3e^{-\conf}$, for $\conf > 0$,
  \begin{align*}
  v(\pi)
  &\geq
    \wh{v}\DRl(\pi)
    -\sqrt{\frac{2x}{n} \, \Var(\wh{v}\DRl(\pi) \mid X_1^n)}
    - \frac{7}{3} \pr{1 + \frac{1}{\lambda}} \frac{x}{n-1}\\
  &- \frac1n \sum_{k=1}^n \sum_{a \in [K]} \pi(a | X_k)
    \pr{
    \abs{ \frac{\pi_b(a \mid X_k)}{\pi_b(a \mid X_k) + \lambda} - 1}
    +
    \eta(a | X_k) \pr{1 - \frac{\pi(a \mid X_k)}{\pi_b(a \mid X_k) + \lambda}}
    }
    - \sqrt{\frac{x}{2 n}}~.
\end{align*}
  and the variance of the estimator is defined as
  \begin{equation}
    \Var(\wh{v}\DRl(\pi)) = \frac{1}{n (n-1)} \sum_{1 \leq i < j \leq n} \pr{Z_i - Z_j}^2
  \end{equation}
  where $Z_i = W_i^{\lambda} (R_i - \eta(X_i, A_i)) + \sum_{a \in [K]} \pi(a | X_i) \eta(a, X_i)$.
  }
\begin{proof}
  We follow the path in as in the proof of \cref{prop:IS_lambda_bernstein} with minor modifications.
  Once again, considering the decomposition
\begin{align*}
  v(\pi) - \wh{v}\DRl(\pi)
  =
  \underbrace{
  v(\pi) - \E\br{\wh{v}\DRl(\pi) \mid X_1^n}
  }_{\mathrm{Bias}}
  +
  \underbrace{
  \E\br{\wh{v}\DRl(\pi) \mid X_1^n} - \wh{v}\DRl(\pi)
  }_{\mathrm{Concentration}}~.
\end{align*}
and observing that $W_k^{\lambda} \leq 1/\lambda$,
the concentration term is bounded by \cref{thm:empirical_bernstein} with $C = 1 + 1/\lambda$ assuming that $\|\eta\|_{\infty} \leq 1$, that is:
\[
  \E\br{\wh{v}\DRl(\pi) \mid X_1^n} - \wh{v}\DRl(\pi)
  \geq
  \sqrt{\frac{2x}{n} \, \Var(\wh{v}\DRl(\pi) \mid X_1^n)}
  + \frac{7}{3} \pr{1 + \frac{1}{\lambda}} \frac{x}{n-1}~.
\]
Now we focus on the bias term which is further decomposed as follows:
\begin{align*}
  v(\pi) - \E\br{\wh{v}\DRl(\pi) \pmid X_1^n}
  =
  v(\pi) - \frac1n \sum_{k=1}^n v(\pi | X_k) + \frac1n \sum_{k=1}^n v(\pi | X_k) - \E\br{\wh{v}\DRl(\pi) \pmid X_1^n}
\end{align*}
As in the proof of \cref{prop:IS_lambda_bernstein} w.p.\ at least $1-e^{-x}, x \geq 0$ we have
\begin{align*}
  \frac1n \sum_{k=1}^n v(\pi \pmid X_k) - v(\pi) \leq \sqrt{\frac{x}{2 n}}~.
\end{align*}
Finally,
\begin{align*}
  &\E\br{\wh{v}\DRl(\pi) \bmid X_1^n} - \frac1n \sum_{k=1}^n v(\pi \pmid X_k)\\
  &=
    \frac1n \sum_{k=1}^n \pr{ \E\br{ W_k^{\lambda} (R_k - \eta(X_k, A_k)) - W_k R_k \bmid X_k} + \sum_{a \in [K]} \pi(a | X_k) \eta(a, X_k)
    }\\
  &=
    \frac1n \sum_{k=1}^n \pr{ \E\br{ (W_k^{\lambda} - W_k) R_k - W_k^{\lambda} \eta(X_k, A_k)  \bmid X_k} + \sum_{a \in [K]} \pi(a | X_k) \eta(a, X_k)
    }\\
  &=
    \frac1n \sum_{k=1}^n \sum_{a \in [K]} \pi(a | X_k) \pr{ \frac{\pi_b(a \mid X_k)}{\pi_b(a \mid X_k) + \lambda} - 1} r(X_k, a) \\
    &+
    \frac1n \sum_{k=1}^n \sum_{a \in [K]} \pi(a | X_k) \eta(a | X_k) \pr{1 - \frac{\pi(a \mid X_k)}{\pi_b(a \mid X_k) + \lambda}}\\
  &\leq
    \frac1n \sum_{k=1}^n \sum_{a \in [K]} \pi(a | X_k)
    \pr{
    \abs{ \frac{\pi_b(a \mid X_k)}{\pi_b(a \mid X_k) + \lambda} - 1}
    +
    \eta(a | X_k) \pr{1 - \frac{\pi(a \mid X_k)}{\pi_b(a \mid X_k) + \lambda}}
    }
\end{align*}
Putting all together and applying a union bound we get the statement w.p.\ at least $1-3 e^{-x}$.
\end{proof}

\section{Additional Experimental Details}
\label{sec:exp_details}

\subsection{Policies.}
\paragraph{Parametrized oracle-based policies. }
For a given dataset $((\bx_i, y_i))_{i=1}^n \subset (\sX \times \sY)^n$, we assume we have access to an oracle $\rho:\sX \to \sY$ that maps contexts to their true label\footnote{In general, this oracle has to be learnt, see discussions on datasets below.}.  
We define an ideal Gibbs policy as
$\pi^{\text{ideal}}(y \mid \bx) \propto e^{\frac{1}{\tau} \mathbb{I}\{y=\rho(\bx)\}}$
and $\tau > 0$ is a temperature parameter.
The smaller $\tau$ is, the more peaky is the distribution on the predicted label.
To create mismatching policies, we consider a \emph{faulty} policy type for which the peak is shifted to another, wrong action for a set of faulty actions $F \subset [K]$ (i.e., if $\rho(\bx) \in F$, the peak is shifted by $1$ cyclically), that is, a faulty policy $\pi^{\text{faulty}(F)}$ is the same as the ideal policy when $\rho(\bx) \not\in F$, and it has distribution
$\pi^{\text{faulty}(F)}(y \mid \bx) \propto e^{\frac{1}{\tau} \mathbb{I}\{y-1= \rho(\bx) \bmod K\}}$.

In the following we consider faulty behavior policies, while one among the target policies is \emph{ideal}.

\paragraph{Learnt policies. }
\label{sec:exp_details:learnt_policies}
There is an important literature on off-policy learning \citep{swaminathan2015batch,swaminathan2015counterfactual,joachims2018deep} that considers the problem of directly learning a policy from logged bandit feedback. These algorithms minimize a loss defined by either \ac{IS} or \ac{WIS} on a parametrized family of policy. 
We implements those two type of parametrized policies as follows: we introduce
$\pi^{\bTheta}(y=k \mid \bx) \propto e^{\frac{1}{\tau} \bx \tp \btheta_k}$ with two choices of parameters given by the optimization problems:
$
  \bhTheta\subIS \in \argmax_{\bTheta \in \reals^{d \times K}} \vh\IS(\pi^{\bTheta})~,  
  \bhTheta\subWIS \in \argmax_{\bTheta \in \reals^{d \times K}} \vh\WIS(\pi^{\bTheta})
$.
In practice we obtain these by running gradient descent with step size $0.01$ for $10^5$ steps.
In all cases the temperature is set to $\tau=0.1$.

\subsection{Datasets and oracles}

\paragraph{Synthetic dataset. } To allow for a precise control of the distribution of the contexts, as well as of the sample size, we generate an underlying multiclass classification problem
through the scikit-learn function \verb+make_classification()+ 
\footnote{See \href{https://scikit-learn.org/stable/modules/generated/sklearn.datasets.make_classification.html}{scikit-learn documentation}}.
Then we obtain a ground truth oracle by training a classifier $\wh{r}$ with a regularized logistic regression (with hyperparameter tuned on the validation set).

\paragraph{Real Datasets.  }
The chosen 8 datasets (see Table~\ref{tab:datasets} in \cref{sec:exp_details}) are loaded from OpenML \citep{dua2017openml}, using scikit-learn \citep{scikit-learn}.
To simplify and stabilize the Gibbs policy construction process, we use the true labels as the peaks of the Gibbs oracle.
In the literature on off-policy evaluation, some experimental settings rely on a ground truth function, which is a multi-class classifier learned on a held-out full-information dataset. This ground truth then replaces the true labels in the policies. Depending on the accuracy of the learnt function, this might naturally induce noise in the policies by having them make mistakes due to a relatively bad oracle. 
Note that in the case of synthetic datasets, it is easy and costless to generate a large train set, get a highly accurate classifier, and discard this data. 
However, for real datasets, the more data is used for training the oracle, the less is available to generate a logged dataset and perform the actual off-policy evaluation experiments. 

While this moves the process away from practice, it has the advantage of allowing a precise control of the values of the policies we create. This is a key point to design stable and reproducible experiments. Learning perfectly interpolating classifiers would lead to the same results, except for the time spent and the data used to do so. 
\paragraph{Baselines.}
In addition to the confidence bound discussed in~\cref{sec:baselines} we consider the standard \ac{DR} estimator and the recent estimation algorithm of~\cite{karampatziakis2019empirical} based on \acf{EL}.
For \ac{DR} (and $\lambda$-DR), rewards are modeled by a ridge regressor (one per class) where a hyperparameter is tuned by a $10$-fold cross-validation (leave-one-out cross-validation for sample size $\leq 100$).
For both $\lambda$-\ac{IS} and $\lambda$-DR, $\lambda$ is set to $1/\sqrt{n}$.
\begin{table}[H]
    \centering
    \small{
    \begin{tabular}{c|c|c|c|c|c|c|c|c|}
         name &   Yeast & PageBlok & OptDigits & SatImage & isolet & PenDigits & Letter & kropt \\
         \hline
         OpenML ID & 181 &  30 & 28 & 182 & 300 & 32 & 6 & 184 \\
         \hline
         Size & 1484 & 5473 & 5620 & 6435 & 7797 & 10,992 & 20,000 & 28,056 \\
    \end{tabular}
    }
    \caption{Real Datasets used in experiments}
    \label{tab:datasets}    
  \end{table}

 \paragraph{Empirical coverage analysis: the case of the Empirical Likelihood estimator. }

 We run the same experiment as that presented in Figure~\ref{fig:coverage-analysis} to study the tightness of the returned lower bound for each estimator: the Gibbs temperature is $\tau=0.3$ and the sample size is $N=1000$ (new dataset for each run), so that the Effective Sample Size is on average $650 \pm 10$. Results are shown on Figure~\ref{fig:EL-tests}. These simulations highlight two interesting facts that make \ac{EL} a slightly different solution to our problem than all other state-of-the-art estimators. First, the returned lower bound is always very close to the true value, and on average just slightly under it. But while this should be a perfect property for our task, the returned value also suffers from quite a large variance such that in many runs the lower bound is larger than the true value (the confidence interval is violated). This seems to indicate that our setting has not yet reached the asymptotic regime in which the confidence interval should have a coverage probability close to $1-\delta$. We conjecture that this may explain the bad performance of EL in our experiments on data (see Section~\ref{sec:experiments}).
 
 \begin{figure}
     \centering
     \includegraphics[width=7cm]{emp_cov_EL_N50-better-annotated.pdf}
     \includegraphics[width=7cm]{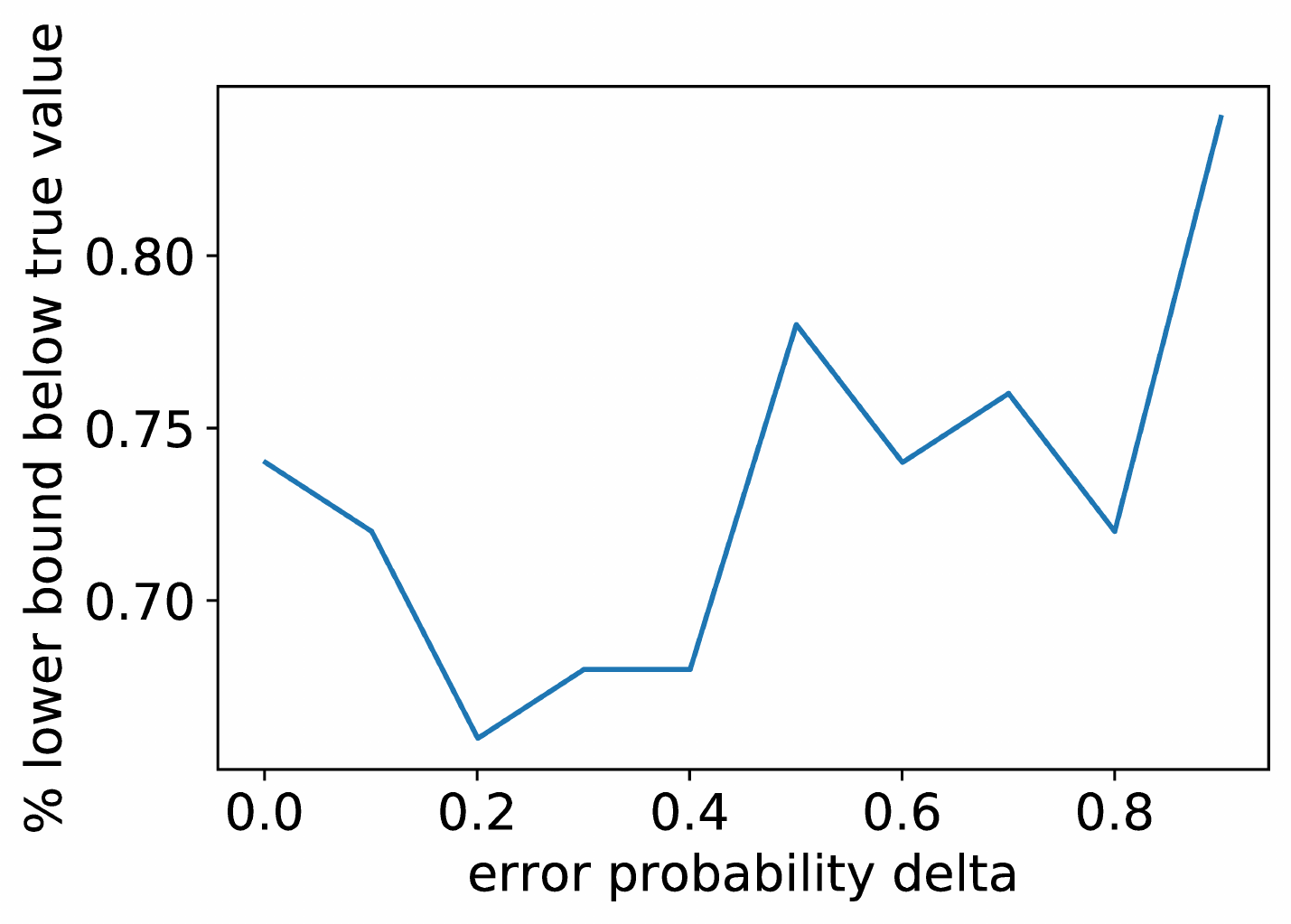}
     \caption{\textbf{Left:}Empirical tightness of the EL lower bound estimator. The returned value is on average very close to the true value. \textbf{Right:} Average rate of violated confidence interval: the empirical coverage is much worse than the true one ($\delta$). Results averaged over 50 runs}
     \label{fig:ap-EL-tests}
 \end{figure}

\section{Implementation of \cref{alg:v_WIS2}}
\label{sec:listing}

In this section we provide a code listing in Python for computing the bound of \cref{thm:WIS_contextual_sim_lb_value}.
In particular, the function $\verb!eslb(...)!$ implements computation of the bound through the Monte-Carlo simulation described in \cref{alg:v_WIS2}.
The function \verb!eslb(t_probs, b_probs, weights, rewards, delta, n_iterations, n_batch_size)! takes $7$ arguments: \verb!t_prob! and \verb!b_prob! are $[0,1]^{n \times K}$ matrices where the $i$-th row corresponds to $\pi(\cdot | X_i)$ and $\pi_b(\cdot | X_i)$ respectively.
Next, \verb!weights! is an vector of importance weights belonging to $\reals_+^n$, similarly \verb!rewards! is a reward vector in $[0,1]^n$, and $\delta \in (0,1)$ is an error probability (recall that the lower bound holds with probability at least $1-\delta$).
Finally, \verb!n_iterations! and \verb!n_batch_size! are Monte-Carlo iterations and the sample size (batch size) used in the simulation (larger \verb!n_batch_size! requires more memory but ensures faster convergence of the simulation).
\verb!eslb(...)! returns a Python dictionary holding $5$ enries: entry \verb!lower_bound! corresponds to the actual lower bound computed according to \cref{thm:WIS_contextual_sim_lb_value}; \verb!est_value! is $\vh(\pi)$, \verb!concentration! is a concentration term denoted by $\epsilon$ in \cref{thm:WIS_contextual_sim_lb_value}, \verb!mult_bias! is a multiplicative bias denoted by $B$, and \verb!concentration_of_contexts! is a $\sqrt{x /(2 n)}$ term.

\vfill

\lstset{ 
upquote=true,
columns=fullflexible,
basicstyle=\scriptsize\ttfamily,
literate={*}{{\char42}}1
         {-}{{\char45}}1
         {\ }{{\copyablespace}}1
}

\newcommand{\copyablespace}{\BeginAccSupp{method=hex,unicode,ActualText=00A0}\hphantom{x}\EndAccSupp{}}

\pagebreak
\begin{lstlisting}[language=Python,caption={Computation of the bound of \cref{thm:WIS_contextual_sim_lb_value}: ``eslb(...)'' function.}]
# Copyright 2020 DeepMind Technologies Limited.
#  
#  
# Licensed under the Apache License, Version 2.0 (the "License");
# you may not use this file except in compliance with the License.
# You may obtain a copy of the License at
#  
# https://www.apache.org/licenses/LICENSE-2.0
#  
# Unless required by applicable law or agreed to in writing, software
# distributed under the License is distributed on an "AS IS" BASIS,
# WITHOUT WARRANTIES OR CONDITIONS OF ANY KIND, either express or implied.
# See the License for the specific language governing permissions and
# limitations under the License.

from __future__ import division
from math import sqrt, log as ln
import numpy as np

def sample_from_simplices_m_times(p, m):
    """ Sample from n probability simplices m times.

    p -- n times K (matrix where earch row describes a probability simplex)
    m -- number of times to sample

    Returns n-times-m matrix of indices of simplex corners.
    """
    axis = 1
    r = np.expand_dims(np.random.rand(p.shape[1-axis], m), axis=axis)
    p_ = np.expand_dims(p.cumsum(axis=axis), axis=2)
    return (np.repeat(p_, m, axis=2) > r).argmax(axis=1)

def eslb(t_probs, b_probs, weights, rewards, delta, n_iterations, n_batch_size):
    """ Computes Efron-Stein lower bound of Theorem 1 as described in Algorithm 1.
    Here n is a sample size, while K is a number actions.

    t_probs -- n-times-K matrix, where $i$-th row correponds to $\pi(\cdot | X_i)$
    b_probs -- n-times-K matrix, where $i$-th row correponds to $\pi_b(\cdot | X_i)$
    weights -- n-sized vector of importance weights
    rewards -- n-sized reward vector
    delta -- error probability in (0,1)
    n_iterations -- Monte-Carlo simulation iterations
    n_batch_size -- Monte-Carlo simulation batch size

    Returns dictionary with 5 enries: lower_bound corresponds to the actual lower bound;
    est_value is an empirical value, concentration is a concentration term, mult_bias
    is a multiplicative bias, and while concentration_of_contexts is a term responsible
    for concentration of contexts.
    """
    conf = ln(2.0/delta)      
    n = len(weights)
    ix_1_n = np.arange(n)
    W_cumsum = weights.cumsum()
    W_cumsum = np.repeat(np.expand_dims(W_cumsum, axis=1), n_batch_size, axis=1)
    W = np.repeat(np.expand_dims(weights, axis=1), n_batch_size, axis=1)
    
    weight_table = t_probs / b_probs
     
    V_unsumed = np.zeros((n,))
    E_V_unsumed = np.zeros((n,))
    E_loo_recip_W = 0.0

    for i in range(n_iterations):
        A_sampled = sample_from_simplices_m_times(b_probs, n_batch_size)
        W_sampled = weight_table[ix_1_n, A_sampled.T].T
        W_sampled_cumsum = W_sampled[::-1, :].cumsum(axis=0)[::-1, :]
        Z = np.copy(W_cumsum)
        Z[:-1, :] += W_sampled_cumsum[1:, :]

        A_sampled_for_U = sample_from_simplices_m_times(b_probs, n_batch_size)
        W_sampled_for_U = weight_table[ix_1_n, A_sampled_for_U.T].T
        A_sampled_for_B = sample_from_simplices_m_times(b_probs, n_batch_size)
        W_sampled_for_B = weight_table[ix_1_n, A_sampled_for_B.T].T        

        Z_repk = Z - W + W_sampled
        W_tilde = W / Z
        U_tilde = W_sampled_for_U / Z_repk
        V_t = (W_tilde + U_tilde)**2

        E_V_new_item = ((W_sampled / W_sampled.sum(axis=0))**2).mean(axis=1)
        V_new_item = V_t.mean(axis=1)

        E_V_unsumed += (E_V_new_item - E_V_unsumed) / (i+1)
        V_unsumed += (V_new_item - V_unsumed) / (i+1)
        
        loo_sum_W = np.outer(np.ones((n,)), np.sum(W_sampled_for_B, axis=0)) - W_sampled_for_B
        recip_min_loo_sum_W = 1.0 / np.min(loo_sum_W, axis=0)
        E_loo_recip_W += (recip_min_loo_sum_W.mean() - E_loo_recip_W) / (i+1)
                
    V = V_unsumed.sum()
    E_V = E_V_unsumed.sum()

    eff_N = 1.0 / E_loo_recip_W
    mult_bias = min(1.0, eff_N / n)
    concentration = sqrt(2.0 * (V + E_V) * (conf + 0.5 * ln(1 + V/E_V)))
    concentration_of_contexts = sqrt(conf / (2*n))
    est_value = weights.dot(rewards) / weights.sum()
    lower_bound = mult_bias * (est_value - concentration ) - concentration_of_contexts

    return dict(lower_bound=max(0, lower_bound), est_value=est_value, concentration=concentration,
                mult_bias=mult_bias, concentration_of_contexts=concentration_of_contexts)  
\end{lstlisting}

\section{Value Bound Decomposition}
\label{sec:decomp}
In this section we present the decomposition of each confidence bound on the value for various estimators evaluated in \cref{sec:experiments}.
In particular, the decomposition is done w.r.t.\ the respective lower bounds on the concentration, bias, and concentration of contexts terms:
\begin{align*}
  v(\pi) - \wh{v}(\pi)
  =
  \underbrace{
  v(\pi) - \E\br{v(\pi) \,|\, X_1^n}
  }_{\text{Concentration of contexts}}
  +
  \underbrace{
  \E\br{v(\pi) \,|\, X_1^n} - \E\br{\wh{v}(\pi) \mid X_1^n}
  }_{\text{Bias}}
  +
  \underbrace{
  \E\br{\wh{v}(\pi) \mid X_1^n} - \wh{v}(\pi)
  }_{\text{Concentration}}~.
\end{align*}
In the following tables each term is presented w.r.t. three target policies discussed in \cref{sec:summary_exp_setting}:
That is \verb!Ideal! is $\pi^{\text{ideal}}$, while \verb!Gibbs-fitted-IW! is $\pi^{\bhTheta\subIS}$, and \verb!Gibbs-fitted-SN! is $\pi^{\bhTheta\subWIS}$.
\subsection{Synthetic Dataset}

\begin{table}[H]
  \centering
  \caption{Concentration term $\eps$ for different confidence intervals and target policies.}
\begin{tabular}{|c|c|c|c|}
  \hline
  Concentration & 5000 & 10000 & 20000\\
  \hline
  \ac{ESLB}: Ideal & 0.680 $\pm$ 0.061 & 0.497 $\pm$ 0.019 & 0.346 $\pm$ 0.013\\
  \ac{ESLB}: Gibbs-fitted-IW & 0.650 $\pm$ 0.079 & 0.498 $\pm$ 0.018 & 0.363 $\pm$ 0.011\\
  \ac{ESLB}: Gibbs-fitted-SN & 0.770 $\pm$ 0.068 & 0.561 $\pm$ 0.029 & 0.378 $\pm$ 0.017\\
  \hline
  $\lambda$-\ac{IS}: Ideal & 0.346 $\pm$ 0.023 & 0.271 $\pm$ 0.008 & 0.206 $\pm$ 0.007\\
  $\lambda$-\ac{IS}: Gibbs-fitted-IW & 0.350 $\pm$ 0.024 & 0.274 $\pm$ 0.008 & 0.208 $\pm$ 0.007\\
  $\lambda$-\ac{IS}: Gibbs-fitted-SN & 0.348 $\pm$ 0.024 & 0.273 $\pm$ 0.008 & 0.207 $\pm$ 0.007\\
  \hline
  Cheb-\ac{WIS}: Ideal & 5.437 $\pm$ 0.000 & 3.242 $\pm$ 0.000 & 2.030 $\pm$ 0.000\\
  Cheb-\ac{WIS}: Gibbs-fitted-IW & 4.006 $\pm$ 0.438 & 2.969 $\pm$ 0.086 & 2.034 $\pm$ 0.026\\
  Cheb-\ac{WIS}: Gibbs-fitted-SN & 6.991 $\pm$ 0.227 & 3.982 $\pm$ 0.122 & 2.344 $\pm$ 0.027\\
  \hline
  DR: Ideal & - & - & -\\
  DR: Gibbs-fitted-IW & - & - & -\\
  DR: Gibbs-fitted-SN & - & - & -\\
  \hline
  $\lambda$-DR: Ideal & 0.412 $\pm$ 0.018 & 0.305 $\pm$ 0.018 & 0.218 $\pm$ 0.009\\
  $\lambda$-DR: Gibbs-fitted-IW & 0.435 $\pm$ 0.017 & 0.310 $\pm$ 0.023 & 0.228 $\pm$ 0.006\\
  $\lambda$-DR: Gibbs-fitted-SN & 0.416 $\pm$ 0.030 & 0.310 $\pm$ 0.013 & 0.210 $\pm$ 0.010\\
  \hline
  Emp.Lik. Ideal & - & - & -\\
  Emp.Lik. Gibbs-fitted-IW & - & - & -\\
  Emp.Lik. Gibbs-fitted-SN & - & - & -\\
  \hline
\end{tabular}
\end{table}
\vfill
\pagebreak

\begin{table}[H]
  \centering
  \caption{Bias term $B$ for different confidence intervals and target policies.}
\begin{tabular}{|c|c|c|c|}
  \hline
  Bias (multiplicative for \ac{WIS}-based CIs) & 5000 & 10000 & 20000\\
  \hline
  \ac{ESLB}: Ideal & 0.988 $\pm$ 0.000 & 0.994 $\pm$ 0.000 & 0.997 $\pm$ 0.000\\
  \ac{ESLB}: Gibbs-fitted-IW & 0.992 $\pm$ 0.001 & 0.995 $\pm$ 0.000 & 0.997 $\pm$ 0.000\\
  \ac{ESLB}: Gibbs-fitted-SN & 0.984 $\pm$ 0.001 & 0.992 $\pm$ 0.001 & 0.996 $\pm$ 0.000\\
  \hline
  $\lambda$-\ac{IS}: Ideal & 0.293 $\pm$ 0.000 & 0.261 $\pm$ 0.000 & 0.219 $\pm$ 0.000\\
  $\lambda$-\ac{IS}: Gibbs-fitted-IW & 0.212 $\pm$ 0.025 & 0.232 $\pm$ 0.008 & 0.217 $\pm$ 0.004\\
  $\lambda$-\ac{IS}: Gibbs-fitted-SN & 0.393 $\pm$ 0.011 & 0.347 $\pm$ 0.014 & 0.272 $\pm$ 0.005\\
  \hline
  Cheb-\ac{WIS}: Ideal & 0.599 $\pm$ 0.000 & 0.715 $\pm$ 0.000 & 0.800 $\pm$ 0.000\\
  Cheb-\ac{WIS}: Gibbs-fitted-IW & 0.671 $\pm$ 0.024 & 0.733 $\pm$ 0.006 & 0.800 $\pm$ 0.002\\
  Cheb-\ac{WIS}: Gibbs-fitted-SN & 0.538 $\pm$ 0.008 & 0.671 $\pm$ 0.007 & 0.776 $\pm$ 0.002\\
  \hline
  DR: Ideal & - & - & -\\
  DR: Gibbs-fitted-IW & - & - & -\\
  DR: Gibbs-fitted-SN & - & - & -\\
  \hline
  $\lambda$-DR: Ideal & 0.515 $\pm$ 0.059 & 0.540 $\pm$ 0.064 & 0.590 $\pm$ 0.046\\
  $\lambda$-DR: Gibbs-fitted-IW & 0.430 $\pm$ 0.095 & 0.570 $\pm$ 0.063 & 0.679 $\pm$ 0.045\\
  $\lambda$-DR: Gibbs-fitted-SN & 0.767 $\pm$ 0.113 & 0.744 $\pm$ 0.086 & 0.790 $\pm$ 0.038\\
  \hline
  Emp.Lik. Ideal & - & - & -\\
  Emp.Lik. Gibbs-fitted-IW & - & - & -\\
  Emp.Lik. Gibbs-fitted-SN & - & - & -\\
  \hline
\end{tabular}
\end{table}

\begin{table}[H]
  \centering
  \caption{Concentration of contexts term for different confidence intervals and target policies.}
\begin{tabular}{|c|c|c|c|}
  \hline
  Concentration of contexts & 5000 & 10000 & 20000\\
  \hline
  \ac{ESLB}: Ideal & 0.025 $\pm$ 0.000 & 0.018 $\pm$ 0.000 & 0.013 $\pm$ 0.000\\
  \ac{ESLB}: Gibbs-fitted-IW & 0.025 $\pm$ 0.000 & 0.018 $\pm$ 0.000 & 0.013 $\pm$ 0.000\\
  \ac{ESLB}: Gibbs-fitted-SN & 0.025 $\pm$ 0.000 & 0.018 $\pm$ 0.000 & 0.013 $\pm$ 0.000\\
  \hline
  $\lambda$-\ac{IS}: Ideal & 0.026 $\pm$ 0.000 & 0.018 $\pm$ 0.000 & 0.013 $\pm$ 0.000\\
  $\lambda$-\ac{IS}: Gibbs-fitted-IW & 0.026 $\pm$ 0.000 & 0.018 $\pm$ 0.000 & 0.013 $\pm$ 0.000\\
  $\lambda$-\ac{IS}: Gibbs-fitted-SN & 0.026 $\pm$ 0.000 & 0.018 $\pm$ 0.000 & 0.013 $\pm$ 0.000\\
  \hline
  Cheb-\ac{WIS}: Ideal & 0.300 $\pm$ 0.000 & 0.212 $\pm$ 0.000 & 0.150 $\pm$ 0.000\\
  Cheb-\ac{WIS}: Gibbs-fitted-IW & 0.300 $\pm$ 0.000 & 0.212 $\pm$ 0.000 & 0.150 $\pm$ 0.000\\
  Cheb-\ac{WIS}: Gibbs-fitted-SN & 0.300 $\pm$ 0.000 & 0.212 $\pm$ 0.000 & 0.150 $\pm$ 0.000\\
  \hline
  DR: Ideal & - & - & -\\
  DR: Gibbs-fitted-IW & - & - & -\\
  DR: Gibbs-fitted-SN & - & - & -\\
  \hline
  $\lambda$-DR: Ideal & 0.037 $\pm$ 0.000 & 0.026 $\pm$ 0.000 & 0.018 $\pm$ 0.000\\
  $\lambda$-DR: Gibbs-fitted-IW & 0.037 $\pm$ 0.000 & 0.026 $\pm$ 0.000 & 0.018 $\pm$ 0.000\\
  $\lambda$-DR: Gibbs-fitted-SN & 0.037 $\pm$ 0.000 & 0.026 $\pm$ 0.000 & 0.018 $\pm$ 0.000\\
  \hline
  Emp.Lik. Ideal & - & - & -\\
  Emp.Lik. Gibbs-fitted-IW & - & - & -\\
  Emp.Lik. Gibbs-fitted-SN & - & - & -\\
  \hline
\end{tabular}
\end{table}

\begin{table}[H]
  \centering
  \caption{Empirical value $\vh$ for different estimators and target policies.}
\begin{tabular}{|c|c|c|c|}
  \hline
  $\vh(\pi)$ & 5000 & 10000 & 20000\\
  \hline
  \ac{ESLB}: Ideal & 0.973 $\pm$ 0.005 & 0.974 $\pm$ 0.002 & 0.974 $\pm$ 0.002\\
  \ac{ESLB}: Gibbs-fitted-IW & 0.822 $\pm$ 0.059 & 0.901 $\pm$ 0.033 & 0.901 $\pm$ 0.016\\
  \ac{ESLB}: Gibbs-fitted-SN & 1.000 $\pm$ 0.000 & 1.000 $\pm$ 0.000 & 0.999 $\pm$ 0.000\\
  \hline
  $\lambda$-\ac{IS}: Ideal & 0.691 $\pm$ 0.044 & 0.727 $\pm$ 0.021 & 0.760 $\pm$ 0.026\\
  $\lambda$-\ac{IS}: Gibbs-fitted-IW & 0.689 $\pm$ 0.045 & 0.732 $\pm$ 0.022 & 0.745 $\pm$ 0.027\\
  $\lambda$-\ac{IS}: Gibbs-fitted-SN & 0.546 $\pm$ 0.040 & 0.587 $\pm$ 0.020 & 0.659 $\pm$ 0.024\\
  \hline
  Cheb-\ac{WIS}: Ideal & 0.973 $\pm$ 0.005 & 0.974 $\pm$ 0.002 & 0.974 $\pm$ 0.002\\
  Cheb-\ac{WIS}: Gibbs-fitted-IW & 0.822 $\pm$ 0.059 & 0.901 $\pm$ 0.033 & 0.901 $\pm$ 0.016\\
  Cheb-\ac{WIS}: Gibbs-fitted-SN & 1.000 $\pm$ 0.000 & 1.000 $\pm$ 0.000 & 0.999 $\pm$ 0.000\\
  \hline
  DR: Ideal & - & - & -\\
  DR: Gibbs-fitted-IW & - & - & -\\
  DR: Gibbs-fitted-SN & - & - & -\\
  \hline
  $\lambda$-DR: Ideal & 0.771 $\pm$ 0.052 & 0.828 $\pm$ 0.039 & 0.879 $\pm$ 0.030\\
  $\lambda$-DR: Gibbs-fitted-IW & 0.769 $\pm$ 0.049 & 0.837 $\pm$ 0.038 & 0.867 $\pm$ 0.032\\
  $\lambda$-DR: Gibbs-fitted-SN & 0.780 $\pm$ 0.066 & 0.791 $\pm$ 0.042 & 0.882 $\pm$ 0.025\\
  \hline
  Emp.Lik. Ideal & - & - & -\\
  Emp.Lik. Gibbs-fitted-IW & - & - & -\\
  Emp.Lik. Gibbs-fitted-SN & - & - & -\\
  \hline
  \hline
\end{tabular}
\end{table}

\subsection{UCI Datasets}

\begin{table}[H]
  \centering
  \caption{Concentration term for different confidence intervals and target policies.}
\resizebox{\textwidth}{!}{
\begin{tabular}{|c|c|c|c|c|c|c|c|c|}
  \hline
  Concentration / Name &    Yeast & PageBlok & OptDigits & SatImage & isolet & PenDigits & Letter & kropt \\
  Size &   1484 & 5473 & 5620 & 6435 & 7797 & 10992 & 20000 & 28056  \\
  \hline
  \ac{ESLB}: Ideal & 0.424 $\pm$ 0.003 & 0.384 $\pm$ 0.038 & 0.672 $\pm$ 0.052 & 0.778 $\pm$ 0.020 & 0.402 $\pm$ 0.041 & 0.494 $\pm$ 0.018 & 0.241 $\pm$ 0.018 & 0.279 $\pm$ 0.006\\
  \ac{ESLB}: Gibbs-fitted-IW & 2.056 $\pm$ 0.082 & 0.839 $\pm$ 0.054 & 0.858 $\pm$ 0.025 & 0.810 $\pm$ 0.088 & 0.827 $\pm$ 0.006 & 0.676 $\pm$ 0.047 & 0.665 $\pm$ 0.013 & 0.595 $\pm$ 0.009\\
  \ac{ESLB}: Gibbs-fitted-SN & 3.060 $\pm$ 0.496 & 0.993 $\pm$ 0.154 & 1.640 $\pm$ 0.204 & 1.242 $\pm$ 0.144 & 1.229 $\pm$ 0.033 & 0.996 $\pm$ 0.148 & 0.833 $\pm$ 0.035 & 0.709 $\pm$ 0.023\\  
  \hline
  $\lambda$-\ac{IS}: Ideal & 0.633 $\pm$ 0.001 & 0.352 $\pm$ 0.020 & 0.385 $\pm$ 0.029 & 0.400 $\pm$ 0.016 & 0.321 $\pm$ 0.021 & 0.301 $\pm$ 0.015 & 0.208 $\pm$ 0.012 & 0.196 $\pm$ 0.006\\
  $\lambda$-\ac{IS}: Gibbs-fitted-IW & 0.821 $\pm$ 0.071 & 0.418 $\pm$ 0.022 & 0.501 $\pm$ 0.015 & 0.435 $\pm$ 0.022 & 0.570 $\pm$ 0.015 & 0.366 $\pm$ 0.009 & 0.347 $\pm$ 0.003 & 0.263 $\pm$ 0.004\\
  $\lambda$-\ac{IS}: Gibbs-fitted-SN & 0.782 $\pm$ 0.056 & 0.393 $\pm$ 0.017 & 0.466 $\pm$ 0.015 & 0.428 $\pm$ 0.024 & 0.483 $\pm$ 0.023 & 0.337 $\pm$ 0.014 & 0.303 $\pm$ 0.006 & 0.243 $\pm$ 0.005\\  
  \hline
  $\lambda$-DR: Ideal & 0.927 $\pm$ 0.004 & 0.484 $\pm$ 0.025 & 0.524 $\pm$ 0.031 & 0.486 $\pm$ 0.017 & 0.430 $\pm$ 0.023 & 0.370 $\pm$ 0.017 & 0.263 $\pm$ 0.015 & 0.242 $\pm$ 0.013\\
  $\lambda$-DR: Gibbs-fitted-IW & 1.179 $\pm$ 0.051 & 0.550 $\pm$ 0.034 & 0.665 $\pm$ 0.027 & 0.555 $\pm$ 0.035 & 0.755 $\pm$ 0.015 & 0.451 $\pm$ 0.019 & 0.427 $\pm$ 0.017 & 0.344 $\pm$ 0.009\\
  $\lambda$-DR: Gibbs-fitted-SN & 0.985 $\pm$ 0.048 & 0.516 $\pm$ 0.024 & 0.544 $\pm$ 0.035 & 0.487 $\pm$ 0.023 & 0.527 $\pm$ 0.037 & 0.378 $\pm$ 0.024 & 0.304 $\pm$ 0.016 & 0.265 $\pm$ 0.010\\  
  \hline
  Cheb-\ac{WIS}: Ideal & 3.133 $\pm$ 0 & 2.464 $\pm$ 0 & 5.450 $\pm$ 0 & 6.852 $\pm$ 0 & 2.405 $\pm$ 0 & 3.258 $\pm$ 0 & 1.268 $\pm$ 0 & 1.541 $\pm$ 0\\
  Cheb-\ac{WIS}: Gibbs-fitted-IW & $-\infty$ & 6.275 $\pm$ 0.994 & 3.720 $\pm$ 0.561 & 5.665 $\pm$ 2.055 & 2.629 $\pm$ 0.305 & 4.375 $\pm$ 0.727 & 5.652 $\pm$ 0.242 & 4.668 $\pm$ 0.126\\
  Cheb-\ac{WIS}: Gibbs-fitted-SN & $-\infty$ & 11.510 $\pm$ 3.473 & 33.712 $\pm$ 4.930 & 17.818 $\pm$ 5.089 & 21.090 $\pm$ 0.994 & 8.806 $\pm$ 1.524 & 6.555 $\pm$ 0.139 & 4.779 $\pm$ 0.085\\  
  \hline
  DR: Ideal & - & - & - & - & - & - & - & -\\
  DR: Gibbs-fitted-IW & - & - & - & - & - & - & - & -\\
  DR: Gibbs-fitted-SN & - & - & - & - & - & - & - & -\\  
  \hline
  Emp.Lik. Ideal & - & - & - & - & - & - & - & -\\
  Emp.Lik. Gibbs-fitted-IW & - & - & - & - & - & - & - & -\\
  Emp.Lik. Gibbs-fitted-SN & - & - & - & - & - & - & - & -\\  
  \hline
\end{tabular}
}
\end{table}

\begin{table}[H]
  \centering
  \caption{Bias term for different confidence intervals and target policies.}
\resizebox{\textwidth}{!}{
\begin{tabular}{|c|c|c|c|c|c|c|c|c|}
\hline
Bias (multiplicative for SN-based CIs) / Name &    Yeast & PageBlok & OptDigits & SatImage & isolet & PenDigits & Letter & kropt \\
Size &   1484 & 5473 & 5620 & 6435 & 7797 & 10992 & 20000 & 28056  \\
  \hline
  \ac{ESLB}: Ideal & 0.996 $\pm$ 0.001 & 0.997 $\pm$ 0 & 0.989 $\pm$ 0.001 & 0.985 $\pm$ 0.001 & 0.997 $\pm$ 0.001 & 0.994 $\pm$ 0.001 & 0.999 $\pm$ 0 & 0.998 $\pm$ 0\\
\ac{ESLB}: Gibbs-fitted-IW & 0.892 $\pm$ 0.018 & 0.987 $\pm$ 0.003 & 0.993 $\pm$ 0.002 & 0.989 $\pm$ 0.005 & 0.996 $\pm$ 0.001 & 0.992 $\pm$ 0.002 & 0.988 $\pm$ 0.001 & 0.990 $\pm$ 0.001\\
\ac{ESLB}: Gibbs-fitted-SN & 0.789 $\pm$ 0.030 & 0.976 $\pm$ 0.005 & 0.951 $\pm$ 0.004 & 0.965 $\pm$ 0.008 & 0.962 $\pm$ 0.003 & 0.979 $\pm$ 0.004 & 0.985 $\pm$ 0.001 & 0.990 $\pm$ 0.001\\
  \hline
  $\lambda$-\ac{IS}: Ideal & 0.054 $\pm$ 0 & 0.067 $\pm$ 0 & 0.170 $\pm$ 0 & 0.244 $\pm$ 0 & 0.076 $\pm$ 0 & 0.151 $\pm$ 0 & 0.056 $\pm$ 0 & 0.097 $\pm$ 0\\
$\lambda$-\ac{IS}: Gibbs-fitted-IW & 0.487 $\pm$ 0.067 & 0.218 $\pm$ 0.033 & 0.118 $\pm$ 0.021 & 0.210 $\pm$ 0.079 & 0.091 $\pm$ 0.015 & 0.238 $\pm$ 0.047 & 0.509 $\pm$ 0.024 & 0.516 $\pm$ 0.016\\
$\lambda$-\ac{IS}: Gibbs-fitted-SN & 0.789 $\pm$ 0.058 & 0.361 $\pm$ 0.076 & 0.718 $\pm$ 0.052 & 0.554 $\pm$ 0.095 & 0.725 $\pm$ 0.007 & 0.506 $\pm$ 0.079 & 0.602 $\pm$ 0.015 & 0.528 $\pm$ 0.012\\
  \hline
  $\lambda$-DR: Ideal & 0.123 $\pm$ 0.013 & 0.115 $\pm$ 0.011 & 0.319 $\pm$ 0.035 & 0.478 $\pm$ 0.040 & 0.313 $\pm$ 0.009 & 0.343 $\pm$ 0.033 & 0.306 $\pm$ 0.007 & 0.294 $\pm$ 0.006\\
$\lambda$-DR: Gibbs-fitted-IW & 1.127 $\pm$ 0.174 & 0.439 $\pm$ 0.074 & 0.488 $\pm$ 0.099 & 0.557 $\pm$ 0.165 & 1.220 $\pm$ 0.077 & 0.706 $\pm$ 0.132 & 1.762 $\pm$ 0.101 & 1.533 $\pm$ 0.070\\
$\lambda$-DR: Gibbs-fitted-SN & 1.443 $\pm$ 0.182 & 0.582 $\pm$ 0.119 & 1.305 $\pm$ 0.103 & 1.154 $\pm$ 0.205 & 1.439 $\pm$ 0.070 & 1.097 $\pm$ 0.157 & 1.395 $\pm$ 0.049 & 1.256 $\pm$ 0.032\\
  \hline
  Cheb-\ac{WIS}: Ideal & 0.722 $\pm$ 0 & 0.767 $\pm$ 0 & 0.599 $\pm$ 0 & 0.543 $\pm$ 0 & 0.772 $\pm$ 0 & 0.714 $\pm$ 0 & 0.865 $\pm$ 0 & 0.841 $\pm$ 0\\
Cheb-\ac{WIS}: Gibbs-fitted-IW & 0 $\pm$ 0 & 0.567 $\pm$ 0.043 & 0.688 $\pm$ 0.031 & 0.601 $\pm$ 0.080 & 0.756 $\pm$ 0.021 & 0.652 $\pm$ 0.038 & 0.590 $\pm$ 0.010 & 0.635 $\pm$ 0.006\\
Cheb-\ac{WIS}: Gibbs-fitted-SN & 0 $\pm$ 0 & 0.426 $\pm$ 0.066 & 0.198 $\pm$ 0.033 & 0.325 $\pm$ 0.062 & 0.279 $\pm$ 0.010 & 0.484 $\pm$ 0.043 & 0.554 $\pm$ 0.005 & 0.630 $\pm$ 0.004\\
  \hline
  DR: Ideal & - & - & - & - & - & - & - & -\\
DR: Gibbs-fitted-IW & - & - & - & - & - & - & - & -\\
DR: Gibbs-fitted-SN & - & - & - & - & - & - & - & -\\  
  \hline
  Emp.Lik. Ideal & - & - & - & - & - & - & - & -\\
Emp.Lik. Gibbs-fitted-IW & - & - & - & - & - & - & - & -\\
Emp.Lik. Gibbs-fitted-SN & - & - & - & - & - & - & - & -\\
  \hline  
\end{tabular}
}
\end{table}

\begin{table}[H]
  \centering
  \caption{Concentration of contexts term for different confidence intervals and target policies.}
\resizebox{\textwidth}{!}{
\begin{tabular}{|c|c|c|c|c|c|c|c|c|}
  \hline
  Concentration of contexts / Name &    Yeast & PageBlok & OptDigits & SatImage & isolet & PenDigits & Letter & kropt \\
  Size &   1484 & 5473 & 5620 & 6435 & 7797 & 10992 & 20000 & 28056  \\
  \hline
  \ac{ESLB}: Ideal & 0.066 $\pm$ 0 & 0.034 $\pm$ 0 & 0.034 $\pm$ 0 & 0.032 $\pm$ 0 & 0.029 $\pm$ 0 & 0.024 $\pm$ 0 & 0.018 $\pm$ 0 & 0.015 $\pm$ 0\\
  \ac{ESLB}: Gibbs-fitted-IW & 0.066 $\pm$ 0 & 0.034 $\pm$ 0 & 0.034 $\pm$ 0 & 0.032 $\pm$ 0 & 0.029 $\pm$ 0 & 0.024 $\pm$ 0 & 0.018 $\pm$ 0 & 0.015 $\pm$ 0\\
  \ac{ESLB}: Gibbs-fitted-SN & 0.066 $\pm$ 0 & 0.034 $\pm$ 0 & 0.034 $\pm$ 0 & 0.032 $\pm$ 0 & 0.029 $\pm$ 0 & 0.024 $\pm$ 0 & 0.018 $\pm$ 0 & 0.015 $\pm$ 0\\  
  \hline
  $\lambda$-\ac{IS}: Ideal & 0.068 $\pm$ 0 & 0.035 $\pm$ 0 & 0.035 $\pm$ 0 & 0.033 $\pm$ 0 & 0.030 $\pm$ 0 & 0.025 $\pm$ 0 & 0.018 $\pm$ 0 & 0.016 $\pm$ 0\\
  $\lambda$-\ac{IS}: Gibbs-fitted-IW & 0.068 $\pm$ 0 & 0.035 $\pm$ 0 & 0.035 $\pm$ 0 & 0.033 $\pm$ 0 & 0.030 $\pm$ 0 & 0.025 $\pm$ 0 & 0.018 $\pm$ 0 & 0.016 $\pm$ 0\\
  $\lambda$-\ac{IS}: Gibbs-fitted-SN & 0.068 $\pm$ 0 & 0.035 $\pm$ 0 & 0.035 $\pm$ 0 & 0.033 $\pm$ 0 & 0.030 $\pm$ 0 & 0.025 $\pm$ 0 & 0.018 $\pm$ 0 & 0.016 $\pm$ 0\\  
  \hline
  $\lambda$-DR: Ideal & 0.096 $\pm$ 0 & 0.050 $\pm$ 0 & 0.049 $\pm$ 0 & 0.046 $\pm$ 0 & 0.042 $\pm$ 0 & 0.035 $\pm$ 0 & 0.026 $\pm$ 0 & 0.022 $\pm$ 0\\
  $\lambda$-DR: Gibbs-fitted-IW & 0.096 $\pm$ 0 & 0.050 $\pm$ 0 & 0.049 $\pm$ 0 & 0.046 $\pm$ 0 & 0.042 $\pm$ 0 & 0.035 $\pm$ 0 & 0.026 $\pm$ 0 & 0.022 $\pm$ 0\\
  $\lambda$-DR: Gibbs-fitted-SN & 0.096 $\pm$ 0 & 0.050 $\pm$ 0 & 0.049 $\pm$ 0 & 0.046 $\pm$ 0 & 0.042 $\pm$ 0 & 0.035 $\pm$ 0 & 0.026 $\pm$ 0 & 0.022 $\pm$ 0\\  
  \hline
  Cheb-\ac{WIS}: Ideal & 0.068 $\pm$ 0 & 0.035 $\pm$ 0 & 0.035 $\pm$ 0 & 0.033 $\pm$ 0 & 0.030 $\pm$ 0 & 0.025 $\pm$ 0 & 0.018 $\pm$ 0 & 0.016 $\pm$ 0\\
  Cheb-\ac{WIS}: Gibbs-fitted-IW & 0.068 $\pm$ 0 & 0.035 $\pm$ 0 & 0.035 $\pm$ 0 & 0.033 $\pm$ 0 & 0.030 $\pm$ 0 & 0.025 $\pm$ 0 & 0.018 $\pm$ 0 & 0.016 $\pm$ 0\\
  Cheb-\ac{WIS}: Gibbs-fitted-SN & 0.068 $\pm$ 0 & 0.035 $\pm$ 0 & 0.035 $\pm$ 0 & 0.033 $\pm$ 0 & 0.030 $\pm$ 0 & 0.025 $\pm$ 0 & 0.018 $\pm$ 0 & 0.016 $\pm$ 0\\  
  \hline
  DR: Ideal & - & - & - & - & - & - & - & -\\
  DR: Gibbs-fitted-IW & - & - & - & - & - & - & - & -\\
  DR: Gibbs-fitted-SN & - & - & - & - & - & - & - & -\\  
  \hline
  Emp.Lik. Ideal & - & - & - & - & - & - & - & -\\
  Emp.Lik. Gibbs-fitted-IW & - & - & - & - & - & - & - & -\\
  Emp.Lik. Gibbs-fitted-SN & - & - & - & - & - & - & - & -\\  
  \hline
\end{tabular}
}
\end{table}

\begin{table}[H]
  \centering
  \caption{Empirical value $\vh$ for different estimators and target policies.}
\resizebox{\textwidth}{!}{
\begin{tabular}{|c|c|c|c|c|c|c|c|c|}
  \hline
  $\vh(\pi)$ / Name &    Yeast & PageBlok & OptDigits & SatImage & isolet & PenDigits & Letter & kropt \\
  Size &   1484 & 5473 & 5620 & 6435 & 7797 & 10992 & 20000 & 28056  \\
  \hline
  \ac{ESLB}: Gibbs-fitted-IW & 0.367 $\pm$ 0.068 & 0.612 $\pm$ 0.075 & 0.510 $\pm$ 0.057 & 0.593 $\pm$ 0.070 & 0.349 $\pm$ 0.019 & 0.604 $\pm$ 0.066 & 0.425 $\pm$ 0.035 & 0.411 $\pm$ 0.023\\
  \ac{ESLB}: Gibbs-fitted-SN & 0.997 $\pm$ 0.004 & 0.919 $\pm$ 0.028 & 0.999 $\pm$ 0.002 & 0.964 $\pm$ 0.035 & 1.000 $\pm$ 0 & 0.969 $\pm$ 0.017 & 0.984 $\pm$ 0.013 & 0.840 $\pm$ 0.053\\
  \ac{ESLB}: Ideal & 0.903 $\pm$ 0.013 & 0.916 $\pm$ 0.007 & 0.907 $\pm$ 0.023 & 0.921 $\pm$ 0.024 & 0.918 $\pm$ 0.005 & 0.908 $\pm$ 0.021 & 0.912 $\pm$ 0.007 & 0.908 $\pm$ 0.010\\
  \hline
  $\lambda$-\ac{IS}: Gibbs-fitted-IW & 0.585 $\pm$ 0.041 & 0.725 $\pm$ 0.040 & 0.963 $\pm$ 0.023 & 0.700 $\pm$ 0.038 & 1.680 $\pm$ 0.068 & 0.772 $\pm$ 0.049 & 0.679 $\pm$ 0.027 & 0.400 $\pm$ 0.019\\
  $\lambda$-\ac{IS}: Gibbs-fitted-SN & 0.175 $\pm$ 0.055 & 0.514 $\pm$ 0.103 & 0.192 $\pm$ 0.070 & 0.352 $\pm$ 0.100 & 0.381 $\pm$ 0.058 & 0.361 $\pm$ 0.101 & 0.297 $\pm$ 0.032 & 0.267 $\pm$ 0.019\\
  $\lambda$-\ac{IS}: Ideal & 0.851 $\pm$ 0.013 & 0.848 $\pm$ 0.013 & 0.750 $\pm$ 0.033 & 0.691 $\pm$ 0.024 & 0.849 $\pm$ 0.014 & 0.772 $\pm$ 0.024 & 0.862 $\pm$ 0.016 & 0.820 $\pm$ 0.010\\
  \hline
  $\lambda$-DR: Ideal & 0.877 $\pm$ 0.019 & 0.870 $\pm$ 0.013 & 0.820 $\pm$ 0.046 & 0.800 $\pm$ 0.038 & 0.875 $\pm$ 0.016 & 0.835 $\pm$ 0.029 & 0.883 $\pm$ 0.013 & 0.834 $\pm$ 0.016\\
  $\lambda$-DR: Gibbs-fitted-IW & 0.682 $\pm$ 0.098 & 0.770 $\pm$ 0.052 & 0.730 $\pm$ 0.082 & 0.710 $\pm$ 0.069 & 0.528 $\pm$ 0.125 & 0.734 $\pm$ 0.054 & 0.477 $\pm$ 0.064 & 0.530 $\pm$ 0.044\\
  $\lambda$-DR: Gibbs-fitted-SN & 0.683 $\pm$ 0.086 & 0.671 $\pm$ 0.102 & 0.583 $\pm$ 0.061 & 0.752 $\pm$ 0.069 & 0.722 $\pm$ 0.071 & 0.743 $\pm$ 0.073 & 0.718 $\pm$ 0.048 & 0.687 $\pm$ 0.039\\  
  \hline
  Cheb-\ac{WIS}: Ideal & 0.903 $\pm$ 0.013 & 0.916 $\pm$ 0.007 & 0.907 $\pm$ 0.023 & 0.921 $\pm$ 0.024 & 0.918 $\pm$ 0.005 & 0.908 $\pm$ 0.021 & 0.912 $\pm$ 0.007 & 0.908 $\pm$ 0.010\\
  Cheb-\ac{WIS}: Gibbs-fitted-IW & 0.367 $\pm$ 0.068 & 0.612 $\pm$ 0.075 & 0.510 $\pm$ 0.057 & 0.593 $\pm$ 0.070 & 0.349 $\pm$ 0.019 & 0.604 $\pm$ 0.066 & 0.425 $\pm$ 0.035 & 0.411 $\pm$ 0.023\\
  Cheb-\ac{WIS}: Gibbs-fitted-SN & 0.997 $\pm$ 0.004 & 0.919 $\pm$ 0.028 & 0.999 $\pm$ 0.002 & 0.964 $\pm$ 0.035 & 1.000 $\pm$ 0 & 0.969 $\pm$ 0.017 & 0.984 $\pm$ 0.013 & 0.840 $\pm$ 0.053\\  
  \hline
  DR: Ideal & - & - & - & - & - & - & - & -\\
  DR: Gibbs-fitted-IW & - & - & - & - & - & - & - & -\\
  DR: Gibbs-fitted-SN & - & - & - & - & - & - & - & -\\  
  \hline
  Emp.Lik. Ideal & - & - & - & - & - & - & - & -\\
  Emp.Lik. Gibbs-fitted-IW & - & - & - & - & - & - & - & -\\
  Emp.Lik. Gibbs-fitted-SN & - & - & - & - & - & - & - & -\\  
  \hline
\end{tabular}
}
\end{table}


\begin{thebibliography}{34}
\providecommand{\natexlab}[1]{#1}
\providecommand{\url}[1]{\texttt{#1}}
\expandafter\ifx\csname urlstyle\endcsname\relax
  \providecommand{\doi}[1]{doi: #1}\else
  \providecommand{\doi}{doi: \begingroup \urlstyle{rm}\Url}\fi

\bibitem[Agarwal et~al.(2017)Agarwal, Basu, Schnabel, and
  Joachims]{agarwal2017effective}
A.~Agarwal, S.~Basu, T.~Schnabel, and T.~Joachims.
\newblock Effective evaluation using logged bandit feedback from multiple
  loggers.
\newblock In \emph{International Conference on Knowledge Discovery and Data
  Mining (KDD)}, 2017.

\bibitem[Bottou et~al.(2013)Bottou, Peters, Qui{\~{n}}onero~Candela, Charles,
  Chickering, Portugaly, Ray, Simard, and Snelson]{BoPe13}
L.~Bottou, J.~Peters, J.~Qui{\~{n}}onero~Candela, D.~X. Charles, M.~Chickering,
  E.~Portugaly, D.~Ray, P.~Y. Simard, and E.~Snelson.
\newblock Counterfactual reasoning and learning systems: the example of
  computational advertising.
\newblock \emph{Journal of Machine Learning Research}, 14\penalty0
  (1):\penalty0 3207--3260, 2013.

\bibitem[Dud{\'\i}k et~al.(2011)Dud{\'\i}k, Langford, and Li]{dudik2011doubly}
M.~Dud{\'\i}k, J.~Langford, and L.~Li.
\newblock Doubly robust policy evaluation and learning.
\newblock In \emph{International Conference on Machine Learing (ICML)}, 2011.

\bibitem[Swaminathan and Joachims(2015{\natexlab{a}})]{swaminathan2015batch}
A.~Swaminathan and T.~Joachims.
\newblock Batch learning from logged bandit feedback through counterfactual
  risk minimization.
\newblock \emph{Journal of Machine Learning Research}, 16\penalty0
  (1):\penalty0 1731--1755, 2015{\natexlab{a}}.

\bibitem[Thomas et~al.(2015{\natexlab{a}})Thomas, Theocharous, and
  Ghavamzadeh]{thomas2015high_a}
P.~S. Thomas, G.~Theocharous, and M.~Ghavamzadeh.
\newblock High-confidence off-policy evaluation.
\newblock In \emph{Conference on Artificial Intelligence (AAAI)},
  2015{\natexlab{a}}.

\bibitem[Thomas et~al.(2015{\natexlab{b}})Thomas, Theocharous, and
  Ghavamzadeh]{thomas2015high_b}
P.~Thomas, G.~Theocharous, and M.~Ghavamzadeh.
\newblock High confidence policy improvement.
\newblock In \emph{International Conference on Machine Learing (ICML)}, pages
  2380--2388, 2015{\natexlab{b}}.

\bibitem[Swaminathan and Joachims(2015{\natexlab{b}})]{swaminathan2015self}
A.~Swaminathan and T.~Joachims.
\newblock The self-normalized estimator for counterfactual learning.
\newblock In \emph{Conference on Neural Information Processing Systems (NIPS)},
  pages 3231--3239, 2015{\natexlab{b}}.

\bibitem[Metelli et~al.(2018)Metelli, Papini, Faccio, and
  Restelli]{metelli2018policy}
A.~M. Metelli, M.~Papini, F.~Faccio, and M.~Restelli.
\newblock Policy optimization via importance sampling.
\newblock In \emph{Conference on Neural Information Processing Systems
  (NeurIPS)}, pages 5442--5454, 2018.

\bibitem[Hesterberg(1995)]{hesterberg1995weighted}
T.~Hesterberg.
\newblock Weighted average importance sampling and defensive mixture
  distributions.
\newblock \emph{Technometrics}, 37\penalty0 (2):\penalty0 185--194, 1995.

\bibitem[Kuzborskij and Szepesv\'ari(2019)]{kuzborskij2019efron}
I.~Kuzborskij and C.~Szepesv\'ari.
\newblock Efron-{S}tein {PAC}-{B}ayesian {I}nequalities.
\newblock arXiv:1909.01931, 2019.
\newblock URL \url{https://arxiv.org/abs/1909.01931}.

\bibitem[Kong(1992)]{kong1992note}
A.~Kong.
\newblock A note on importance sampling using standardized weights.
\newblock University of Chicago, Dept. of Statistics, Tech. Rep, 1992.

\bibitem[Elvira et~al.(2018)Elvira, Martino, and Robert]{elvira2018rethinking}
V.~Elvira, L.~Martino, and C.~P. Robert.
\newblock Rethinking the effective sample size.
\newblock arXiv:1809.04129, 2018.

\bibitem[Mnih et~al.(2008)Mnih, Szepesv{\'a}ri, and
  Audibert]{mnih2008empirical}
V.~Mnih, C.~Szepesv{\'a}ri, and J.-Y. Audibert.
\newblock Empirical bernstein stopping.
\newblock In \emph{International Conference on Machine Learing (ICML)}, 2008.

\bibitem[Maurer and Pontil(2009)]{maurer2009empirical}
A.~Maurer and M.~Pontil.
\newblock Empirical bernstein bounds and sample variance penalization.
\newblock In \emph{Conference on Computational Learning Theory (COLT)}, 2009.

\bibitem[Owen(2013)]{owen2013book}
A.~B. Owen.
\newblock \emph{Monte Carlo theory, methods and examples}.
\newblock 2013.
\newblock URL \url{https://statweb.stanford.edu/~owen/mc/}.

\bibitem[Ionides(2008)]{ionides2008truncated}
E.~L. Ionides.
\newblock Truncated importance sampling.
\newblock \emph{Journal of Computational and Graphical Statistics}, 17\penalty0
  (2):\penalty0 295--311, 2008.

\bibitem[Vehtari et~al.(2015)Vehtari, Simpson, Gelman, Yao, and
  Gabry]{vehtari2015pareto}
A.~Vehtari, D.~Simpson, A.~Gelman, Y.~Yao, and J.~Gabry.
\newblock Pareto smoothed importance sampling.
\newblock \emph{arXiv preprint arXiv:1507.02646}, 2015.

\bibitem[Gilotte et~al.(2018)Gilotte, Calauz{\`e}nes, Nedelec, Abraham, and
  Doll{\'e}]{gilotte2018offline}
A.~Gilotte, C.~Calauz{\`e}nes, T.~Nedelec, A.~Abraham, and S.~Doll{\'e}.
\newblock Offline a/b testing for recommender systems.
\newblock In \emph{International Conference on Web Search and Data Mining},
  2018.

\bibitem[Swaminathan and
  Joachims(2015{\natexlab{c}})]{swaminathan2015counterfactual}
A.~Swaminathan and T.~Joachims.
\newblock Counterfactual risk minimization: Learning from logged bandit
  feedback.
\newblock In \emph{International Conference on Machine Learing (ICML)}, pages
  814--823, 2015{\natexlab{c}}.

\bibitem[Dud{\'\i}k et~al.(2014)Dud{\'\i}k, Erhan, Langford, and
  Li]{dudik2014doubly}
M.~Dud{\'\i}k, D.~Erhan, J.~Langford, and L.~Li.
\newblock Doubly robust policy evaluation and optimization.
\newblock \emph{Statistical Science}, 29\penalty0 (4):\penalty0 485--511, 2014.

\bibitem[Farajtabar et~al.(2018)Farajtabar, Chow, and
  Ghavamzadeh]{farajtabar2018more}
M.~Farajtabar, Y.~Chow, and M.~Ghavamzadeh.
\newblock More robust doubly robust off-policy evaluation.
\newblock In \emph{International Conference on Machine Learing (ICML)}, 2018.

\bibitem[Wang et~al.(2017)Wang, Agarwal, and Dudik]{wang2017optimal}
Y.-X. Wang, A.~Agarwal, and M.~Dudik.
\newblock Optimal and adaptive off-policy evaluation in contextual bandits.
\newblock In \emph{International Conference on Machine Learing (ICML)}, 2017.

\bibitem[Su et~al.(2019{\natexlab{a}})Su, Wang, Santacatterina, and
  Joachims]{su2019cab}
Y.~Su, L.~Wang, M.~Santacatterina, and T.~Joachims.
\newblock Cab: Continuous adaptive blending for policy evaluation and learning.
\newblock In \emph{International Conference on Machine Learing (ICML)}, pages
  6005--6014, 2019{\natexlab{a}}.

\bibitem[Su et~al.(2019{\natexlab{b}})Su, Dimakopoulou, Krishnamurthy, and
  Dud{\'\i}k]{su2019doubly}
Y.~Su, M.~Dimakopoulou, A.~Krishnamurthy, and M.~Dud{\'\i}k.
\newblock Doubly robust off-policy evaluation with shrinkage.
\newblock \emph{arXiv preprint arXiv:1907.09623}, 2019{\natexlab{b}}.

\bibitem[Kallus(2018)]{kallus2018balanced}
N.~Kallus.
\newblock Balanced policy evaluation and learning.
\newblock In \emph{Conference on Neural Information Processing Systems
  (NeurIPS)}, pages 8895--8906, 2018.

\bibitem[Karampatziakis et~al.(2019)Karampatziakis, Langford, and
  Mineiro]{karampatziakis2019empirical}
N.~Karampatziakis, J.~Langford, and P.~Mineiro.
\newblock Empirical likelihood for contextual bandits.
\newblock In \emph{NeurIPS 2019 Optimization Foundations for Reinforcement
  Learning Workshop}, 2019.

\bibitem[Mineiro and Karampatziakis(2019)]{mineiro2019el_code}
P.~Mineiro and N.~Karampatziakis.
\newblock Code to reproduce the results in the paper 'empirical likelihood for
  contextual bandits'.
\newblock
  \url{https://github.com/pmineiro/elfcb/blob/master/MLE/MLE/asymptoticconfidenceinterval.py},
  2019.
\newblock Accessed: 2020-09-30.

\bibitem[Bietti et~al.(2018)Bietti, Agarwal, and
  Langford]{bietti2018contextual}
A.~Bietti, A.~Agarwal, and J.~Langford.
\newblock A contextual bandit bake-off.
\newblock \emph{arXiv preprint arXiv:1802.04064}, 2018.

\bibitem[Joachims et~al.(2018)Joachims, Swaminathan, and
  de~Rijke]{joachims2018deep}
T.~Joachims, A.~Swaminathan, and M.~de~Rijke.
\newblock Deep learning with logged bandit feedback.
\newblock In \emph{International Conference on Learning Representations
  (ICLR)}, 2018.

\bibitem[Dua and Graff(2017)]{dua2017openml}
D.~Dua and C.~Graff.
\newblock {UCI} machine learning repository, 2017.
\newblock URL \url{http://archive.ics.uci.edu/ml}.

\bibitem[Athey and Wager(2021)]{athey2021policy}
S.~Athey and S.~Wager.
\newblock Policy learning with observational data.
\newblock \emph{Econometrica}, 89\penalty0 (1):\penalty0 133--161, 2021.

\bibitem[Dziugaite and Roy(2017)]{dziugaite2017computing}
G.~K. Dziugaite and D.~M. Roy.
\newblock Computing nonvacuous generalization bounds for deep (stochastic)
  neural networks with many more parameters than training data.
\newblock In \emph{Uncertainty in Artificial Intelligence (UAI)}, 2017.

\bibitem[Boucheron et~al.(2013)Boucheron, Lugosi, and
  Massart]{boucheron2013concentration}
S.~Boucheron, G.~Lugosi, and P.~Massart.
\newblock \emph{Concentration inequalities: A nonasymptotic theory of
  independence}.
\newblock Oxford University Press, 2013.

\bibitem[Pedregosa et~al.(2011)Pedregosa, Varoquaux, Gramfort, Michel, Thirion,
  Grisel, Blondel, Prettenhofer, Weiss, Dubourg, Vanderplas, Passos,
  Cournapeau, Brucher, Perrot, and Duchesnay]{scikit-learn}
F.~Pedregosa, G.~Varoquaux, A.~Gramfort, V.~Michel, B.~Thirion, O.~Grisel,
  M.~Blondel, P.~Prettenhofer, R.~Weiss, V.~Dubourg, J.~Vanderplas, A.~Passos,
  D.~Cournapeau, M.~Brucher, M.~Perrot, and E.~Duchesnay.
\newblock Scikit-learn: Machine learning in {P}ython.
\newblock \emph{Journal of Machine Learning Research}, 12:\penalty0 2825--2830,
  2011.

\end{thebibliography}
\end{document}